\documentclass[11pt]{article}
\usepackage{amsmath,amssymb,amsthm}
\usepackage{array}
\usepackage{fullpage}
\usepackage{natbib}
\usepackage{comment}
\usepackage{tikz}

\usepackage[colorlinks=true, linkcolor=blue, citecolor=blue]{hyperref}

\newtheorem{theorem}{Theorem}
\newtheorem{lemma}{Lemma}
\newtheorem{corollary}{Corollary}
\newtheorem{proposition}{Proposition}
\newtheorem{claim}{Claim}
\newtheorem{observation}[theorem]{Observation}

\theoremstyle{remark}
\newtheorem{remark}{Remark}

\theoremstyle{definition}
\newtheorem{definition}{Definition}

\newcommand{\tz}{\mathrm{tz}} 
\newcommand{\Walsh}{\mathrm{Wal}}

\begin{document}

\title{Optimal Lower Bounds for Online Multicalibration}
\author{Natalie Collina \qquad  Jiuyao Lu \qquad Georgy Noarov \qquad Aaron Roth\\{University of Pennsylvania}}
\date{}

\maketitle

\begin{abstract}
We prove tight lower bounds for online multicalibration, establishing an information-theoretic separation from marginal calibration.

In the general setting where group functions can depend on both context and the learner's predictions, we prove an $\Omega(T^{2/3})$ lower bound on expected multicalibration error using just three disjoint binary groups. This matches the upper bounds of \cite{noarov2023high} up to logarithmic factors and exceeds the $O(T^{2/3-\varepsilon})$ upper bound for marginal calibration \citep{dagan2025breaking}, thereby separating the two problems.

We then turn to lower bounds for the more difficult case of group functions that may depend on context but not on the learner's predictions. In this case, we  establish an $\widetilde{\Omega}(T^{2/3})$ lower bound for online multicalibration via an $O(\log^3 T)$-sized group family constructed from an orthonormal basis, again matching upper bounds up to logarithmic factors.
\end{abstract}

 \setcounter{page}{0}
 \thispagestyle{empty}
 \clearpage
 \setcounter{page}{0}
 \clearpage
\newpage
{\small
\tableofcontents
}
 \thispagestyle{empty}
\clearpage
\pagestyle{plain}
 \setcounter{page}{0}
  \clearpage
\pagenumbering{arabic}
\setcounter{page}{1}

\section{Introduction}

\paragraph{Online calibration} A sequence of predictions $p^1,\ldots,p^T$ is \emph{calibrated} to a sequence of outcomes $y^1,\ldots,y^T$ if, informally, the average of the predictions equals the average of the outcomes, even \emph{conditional} on the value of the prediction \citep{dawid1982well}. To measure the deviation from perfect calibration, one can define the cumulative empirical bias conditional on a prediction $v \in \mathbb{R}$ as $B_T(v) = \sum_{t: p^t = v}(p^t - y^t)$. The classical mis-calibration measure known as \emph{expected calibration error (ECE)} sums the magnitude of the empirical bias conditional on each prediction: \[\textrm{Err}_T = \sum_{v \in \{p^1,\ldots,p^T\}}|B_T(v)|.\] 

In a seminal result, \cite{foster1998asymptotic} showed that there exists a randomized algorithm able to generate predictions that are guaranteed to have expected calibration error scaling as $o(T)$ for arbitrary/adversarially selected sequences $y^1,\ldots,y^T$. The optimal \emph{rate} at which calibration error can be guaranteed has been a long-standing open question, which has seen recent partial progress. A long-standing upper bound established that it was possible to obtain calibration error scaling as $O(T^{2/3})$ \citep{foster1998asymptotic,hart2025calibrated,abernethy2011blackwell}. For many years no lower bound better than $\Omega(T^{0.5})$ was known, until \cite{qiao2021stronger} proved a lower bound of $\Omega(T^{0.528})$. The current state of the art, due to \cite{dagan2025breaking}, establishes that the optimal rate for calibration is between $\Omega(T^{0.54389})$ and $O(T^{2/3-\varepsilon})$ for some (extremely small) constant $\varepsilon > 0$. The result of \cite{dagan2025breaking}  was a breakthrough for giving the first upper bound improvement showing that the long-standing $T^{2/3}$ rate was not optimal for marginal calibration.

\paragraph{Online multicalibration} Calibration is on its own a weak guarantee in that it marginalizes over the entire sequence, which substantially limits its applicability in contextual prediction settings. But it is possible to give stronger guarantees, asking for calibration not just marginally, but simultaneously on many different subsequences or weightings of the data that can be defined both by external context and the predictions themselves \citep{dawid1985calibration,lehrer2003approachability,sandroni2003calibration}. 

A modern CS formulation of this idea is called \emph{multicalibration}, introduced by \cite{hebert2018multicalibration}. Multicalibration reweights the residuals of the predictions by ``group functions'', which are simply mappings $g:X \times \mathbb{R} \rightarrow [0,1]$ from any pair $(x, v) =$ (context, learner's prediction) to a bounded weight $g(x,v)$. When $g(x, v)$ is independent of $v$, we will refer to such a group as \emph{prediction-independent}. The group- and prediction-conditional cumulative empirical bias is defined as $B_T(v,g) = \sum_{t : p^t = v}g(x^t,p^t)(p^t-y^t)$, and the group-conditional calibration error is then given by $\textrm{Err}_T(g) = \sum_{v \in \{p^1,\ldots,p^T\}}|B_T(v,g)|$. The multicalibration error with respect to a collection of group functions $G$ is defined as \[\textrm{MCerr}_T(G) = \max_{g \in G}\textrm{Err}_T(g).\] 

Multicalibration and related guarantees have found many applications in recent years, from learning in a loss-function agnostic manner \citep{gopalan2022omnipredictors, gopalan2023loss} to strengthening complexity theoretic constructions \citep{casacuberta2024complexity,dwork2025supersimulators} to low complexity algorithms for distributed information aggregation \citep{collina2025tractable,collina2026collaborative}. Moreover, similar techniques to those that have been used to derive algorithms guaranteeing marginal calibration in sequential adversarial settings have also been adapted to multicalibration, including methods based on multi-objective optimization and Blackwell approachability \citep{gupta2022online,lee2022online,noarov2023high,haghtalab2023unifying}, swap regret minimization \citep{globus2023multicalibration,gopalan2023swap,garg2024oracle}, and defensive forecasting \citep{perdomo2025defense}. 

\paragraph{What are the optimal multicalibration rates?} Just as for marginal calibration, the minimax online multicalibration rate has remained a difficult open challenge. Recently, \cite{noarov2023high} (and independently \cite{ghuge2025improved}) established  that online multicalibration can be obtained at the rate $\widetilde O(T^{2/3}\sqrt{\log |G|})$; very recently, \cite{hu2025efficient} gave corresponding \emph{oracle-efficient} rates (not just for means, but for any elicitable property; c.f.~\cite{noarov2023statistical}). In fact, given their benign $O(\sqrt{\log G})$ dependence on $|G|$, these algorithms guarantee $\widetilde O(T^{2/3})$ rates even for $poly(T)$-sized group families $G$.

However, to date no lower bound for online multicalibration has been obtained, beyond the $\Omega(T^{0.54389})$ lower bound inherited from the easier problem of marginal calibration \citep{dagan2025breaking}. Nor have any $O(T^{2/3-\epsilon})$-rate multicalibration algorithms been derived for any $\epsilon > 0$.

It is natural to wonder whether the answer to the online multicalibration rate problem is simply the (as yet unknown) minimax marginal calibration rate. Indeed, consider any \emph{constant-sized} collection $G$ of \emph{prediction-independent} groups\footnote{For convenience further assume the groups in $G$ are binary-valued, so that the Venn diagram corresponding to $G$ is well-defined for the purposes of this reduction. We note that our lower bounds hold even for binary groups, thus showing that binary groups are as hard as real-valued ones in the minimax sense.} (i.e., groups that only depend on context). Then it is possible to obtain multicalibration at the rate of marginal calibration: one can just instantiate, in parallel, $2^{|G|}$ copies of a minimax-optimal marginal calibration algorithm for all regions in the Venn diagram partition corresponding to $G$ (i.e., for all possible group intersection patterns).

However, this rate-preserving reduction breaks for group collections $G$ that are more complex than just described. First, if the groups in $G$ depend on the learner's predictions, this invalidates the reduction mechanism: on any given round, the set of active groups (and thus the active copy of the marginal calibration algorithm) will not be determined until the learner makes the prediction. Second, if the size of $G$ is not constant but instead grows even logarithmically with $T$, then combining all of the copies' guarantees will incur overhead that will destroy the optimal marginal calibration rate. 

Therefore, we have two challenging regimes in which to pin down the complexity of online multicalibration and to establish whether it is still exactly as hard as marginal calibration or strictly harder: (1) Prediction-dependent group collections; and (2) prediction-independent group collections whose size grows with $T$.
To summarize, in this paper our goal will be to answer the fundamental questions:
\begin{center}
    \emph{What are the minimax rates for sequential multicalibration? \\ Is sequential multicalibration a strictly harder problem than sequential marginal calibration?}
\end{center}



 \begin{table}[t]
  \centering
  \small
  \setlength{\tabcolsep}{4pt}
  \renewcommand{\arraystretch}{1.25}
  \begin{tabular}{@{}>{\raggedright\arraybackslash}p{0.17\linewidth} >{\raggedright\arraybackslash}p{0.21\linewidth} >{\raggedright\arraybackslash}p{0.30\linewidth} >{\raggedright\arraybackslash}p{0.24\linewidth}@{}}
    \hline
    \textbf{Setting} & \textbf{Groups} & \textbf{Upper bounds} & \textbf{Lower bounds} \\
    \hline
    Marginal\newline calibration & none  &
    $O\!\left(T^{2/3-\varepsilon}\right)$ (for some $\varepsilon>0$)\newline
    {\scriptsize\cite{dagan2025breaking}} &
    $\Omega\!\left(T^{0.54389}\right)$\newline
    {\scriptsize\cite{dagan2025breaking}} \\
    \hline
    Multicalibration\newline (general) & prediction-\newline dependent\newline $g(x,v)$ &
    $\tilde O\!\left(T^{2/3}\sqrt{\log|G|}\right)$\newline
    {\scriptsize\cite{noarov2023high}} &
    $\Omega\!\left(T^{2/3}\right)$ (even $|G|=3$ disjoint binary groups) \\
    \hline
    Multicalibration\newline (restricted) & prediction-\newline independent\newline $g(x)$, $|G|=O(1)$ & Rate preserving
    reduction to marginal calibration   &
    (no separation possible from marginal calibration) \\
    \hline
    Multicalibration\newline (restricted) & prediction-\newline independent\newline $g(x)$, $|G|=O(\log^3 T)$ &
    $\tilde O\!\left(T^{2/3}\sqrt{\log|G|}\right)$ \newline {\scriptsize\cite{noarov2023high}}  &
    $\tilde\Omega\!\left(T^{2/3}\right)$ \\
  \end{tabular}
  \caption{\textbf{Summary of regimes and rates.}
  We study online adversarial multicalibration.
  For general prediction-dependent groups $g(x,v)$ we prove an optimal $\Omega(T^{2/3})$ lower bound, separating multicalibration from marginal calibration.
  For prediction-independent groups $g(x)$, constant-sized families reduce to marginal calibration up to a $2^{|G|}$ factor, precluding a separation from marginal calibration. For a group family of size  $|G|=O(\log^3 T)$ we again prove an optimal $\tilde\Omega(T^{2/3})$ lower bound.}
  \label{tab:rates-summary}
\end{table}

\subsection{Our Results}
We answer the above questions in a strong sense: We prove $\widetilde \Omega(T^{2/3})$ online multicalibration lower bounds that match (up to logarithmic factors) the existing $\widetilde O(T^{2/3})$ upper bounds --- both in the prediction-independent case, and in the prediction-dependent case with growing-in-$T$ group family size. Therefore, our results \emph{strictly separate the complexity of marginal calibration from that of multicalibration}, while also separating two regimes of complexity for multicalibration --- \emph{prediction-dependent} and \emph{prediction-independent} groups --- that existing upper bounds have treated identically. 

The  work of \cite{hebert2018multicalibration} defined multicalibration in terms of what we will call binary prediction-independent groups, that depend on the context but not the prediction, and map to a binary range: $g:X\times \mathbb{R}\rightarrow \{0,1\}$ such that for all $v, v' \in \mathbb{R}$ and for all $x \in X$, $g(x,v) = g(x,v')$. Subsequent work generalized the notion of groups to allow them to be weighting functions with range $[0,1]$ rather than binary valued, and to explicitly depend on predictions $p^t$ \citep{gopalan2022low,kim2022universal,deng2023happymap} --- see also \cite{kakade2008deterministic} and \cite{sandroni2003calibration} for earlier work using similar grouping functions.  \cite{sandroni2003calibration} called prediction-dependent binary groups ``forecast based checking rules''. The algorithms that were developed for online multicalibration \citep{gupta2022online,noarov2023high,haghtalab2023unifying} are all able to support this general notion of group functions at the same rates as they support the special case of binary prediction-independent group functions. As foreshadowed by the above discussion, our lower bounds identify a distinction between the general case and prediction independent groupings.
\begin{enumerate}
    \item In Section \ref{sec:prediction-dependent} we show an optimal $\Omega(T^{2/3})$ lower bound for multicalibration in the general case, when groups can be defined in terms of predictions. This matches (up to logarithmic factors) the rate obtained by existing efficient algorithms \citep{noarov2023high,ghuge2025improved,hu2025efficient}, and so establishes the optimal statistical rate for multicalibration in the general case. Because of the $O(T^{2/3-\varepsilon})$ \emph{upper bound} for marginal calibration established by \cite{dagan2025breaking} it also formally separates the statistical complexity of multicalibration from marginal calibration. Our lower bound instance is realized by just 3 binary valued groups that are also disjoint (such that exactly one is active on any given round). Therefore, this lower bound is driven neither by the complexity nor by the intersectionality of the group family, but rather only by their prediction-dependent nature.
    \item We observe in Appendix \ref{sec:pred-independent} that no similar separation is possible for constant-sized families of binary groups if one further restricts them to be  prediction independent. Indeed, there is an extremely simple reduction from multicalibration to marginal calibration for prediction independent groups that, given binary groups $G$, instantiates a marginal calibration algorithm for each of the $2^{|G|}$ possible \emph{intersection patterns} of groups. At each round $t$ exactly one group intersection pattern is realized and (since the groups are prediction independent) we can route the round to the appropriate marginal calibration algorithm. This gives multicalibration at the best rate obtainable for marginal calibration, up to a $2^{|G|}$ factor (slightly more nuanced bounds are possible, and we give these in Appendix \ref{sec:pred-independent}). If $|G|$ is constant valued (independent of $T$), then this exponential-in-$|G|$ blowup is also only constant-valued, and does not affect the asymptotic rate. Hence any separation between marginal calibration and multicalibration for the special case of prediction independent groups must depend on group families with cardinality $|G|$ growing as a function of $T$. We note that \emph{upper bounds} \citep{gupta2022online,noarov2023high,ghuge2025improved, hu2025efficient} depend only logarithmically on $|G|$, so even families of size $|G| = \textrm{poly}(T)$ incur only logarithmic overhead.
    \item In Section \ref{sec:walsh-localrate-nonneg}, we give an optimal $\tilde \Omega(T^{2/3})$ lower bound for multicalibration even restricting to prediction independent binary groups, using a group family of size $|G| = O(\log^3 T)$. This once again matches existing upper bounds up to logarithmic factors \citep{noarov2023high,ghuge2025improved,hu2025efficient} and separates multicalibration (with group families with cardinality depending polylogarithmically on $T$) from marginal calibration \citep{dagan2025breaking}. This is the most technically challenging result of the paper.
\end{enumerate}
We note that all of our lower bound constructions use only \emph{binary} groups, which also establishes that these are already as hard as arbitrary weighted groups (which the upper bounds support). 
\subsection{Proof Overviews}

We now sketch our lower bound constructions and their analyses. Since it turns out that rates of $\tilde \Theta(T^{2/3})$ are the ``right answer'' for multicalibration, it is helpful to understand how these rates arise in upper bounds. This is most easily understood through the ``minimax'' lens of \cite{hart2025calibrated} in which the order-of-play of the learner and the adversary are reversed in the analysis using the minimax theorem. In the reversed order-of-play, the adversary first commits to a (possibly adaptive) strategy mapping histories to distributions over outcomes, and the learner has knowledge of this strategy before it must make  predictions. One option for the learner in this reversed order of play is the ``honest'' strategy that at every round predicts $p^t = \mathbb{E}[y^t]$ which is feasible, given that in this order of play, the learner knows the adversary's strategy. This is not a good strategy on its own, as the calibration error metric sums the magnitude of the empirical bias across all prediction values the learner uses, and it might be that $\mathbb{E}[y^t]$ takes on a distinct value for each $t$; in this case there would be no cancellations and the learner's calibration error would scale linearly with $T$. But the learner could \emph{round} their prediction $p^t \approx \mathbb{E}[y^t]$ to the nearest multiple of $1/m$, which would introduce bias at most $1/m$ at each round and cumulatively at most $T/m$ across all $T$ rounds. Standard (anti)-concentration arguments establish that if there is a value $v$ that the learner predicts $k$ times, then the empirical sum of the labels on the rounds in which $v$ is predicted will differ from its expectation by roughly $\approx \sqrt{k}$. The ``rounded'' honest learner uses at most $m$ different prediction values, and the worst-case for the learner is if they are all used equally frequently: $k \approx T/m$. In this case the summed noise magnitude of the learner's predictions scales as $m\sqrt{k}\approx \sqrt{mT}$. Picking $m$ to trade this noise term off against the bias term of $T/m$ results in a $\tilde O(T^{2/3})$ upper bound. This style of argument applies both to marginal calibration and multicalibration, because the ``rounded honest'' strategy obtains bounds of this form simultaneously on any subsequence of rounds. However, it does not give a lower bound because there may be a strategy for the learner that obtains better calibration error than ``honesty'' by cleverly setting up cancellations --- indeed, this is exactly what \cite{dagan2025breaking} show in the case of marginal calibration. At a very high level, our goal is to rule out that ``dishonest'' strategies can be beneficial. We note in passing a conceptual similarity to recent work on designing truthful calibration measures \citep{haghtalab2024truthfulness,qiao2025truthfulness,hartline2025perfectly}, although our settings are incomparable.

Both lower bounds we prove in Sections \ref{sec:prediction-dependent} and \ref{sec:walsh-localrate-nonneg} share a common pattern. The lower bound instances are both \emph{oblivious}/non-adaptive sequences of context/label pairs $(x^t,y^t)$ such that:
\begin{enumerate}
    \item The labels $y^t$ are independent random variables with $\mathbb{E}[y^t] = x^t$, and
    \item The contexts themselves $x^t$ are uniformly spread out in a grid in $[1/4,3/4]$. 
\end{enumerate}
Instances like this make it possible for the learner to make ``honest predictions'' of $p^t = \mathbb{E}[y^t] = x^t$, because the label mean is communicated to the learner through the context. However, as discussed, if the learner were to engage in this ``honest'' prediction strategy, their predictions (although unbiased) would incur high error because of noise. Our group constructions are designed  to punish dishonest strategies, and thus to force the learner into the high-error ``honest'' regime. For prediction dependent groups, there is a conceptually straightforward way to do this---although there are a number of technical obstacles to carrying the idea through formally. For prediction independent groups this is more complex.

\subsubsection{Lower Bound for the General Case
}
\paragraph{Hard instance.} We use a Bernoulli environment in which contexts $x^t$ cycle over a fixed grid in $[1/4, 3/4]$, and labels are drawn as
$$ y^t \sim \mathrm{Bernoulli}(x^t) $$
independently across time. We choose a grid size $m \approx T^{1/3}$ and a small margin parameter $\eta \approx \sqrt{m/T}$. The construction uses only three disjoint binary prediction-dependent groups that partition the prediction space according to whether the learner made a prediction that was approximately ``honest'', or was dishonest either by predicting substantially above the label mean (an ``overshoot'') or by predicting substantially below the label mean (an ``undershoot''):
\begin{align*}
    g_1(x,v) &= \mathbf{1}[v \ge x+\eta] \quad \text{(large overshoots),} \\
    g_2(x,v) &= \mathbf{1}[v \le x-\eta] \quad \text{(large undershoots),} \\
    g_3(x,v) &= \mathbf{1}[|v-x|<\eta] \quad \text{(approximately honest predictions).}
\end{align*}

\paragraph{Proof steps.} The proof shows that any algorithm must incur multicalibration error $\Omega(T^{2/3})$ on at least one of these groups, by splitting into two complementary cases.

1. \textbf{Partition rounds into big deviations and ``honest'' rounds.}
    For a fixed algorithm and realization, define
    \begin{align*}
        \text{big-deviation rounds } B &:= \{t : |p^t-x^t| \ge \eta\}, \\
        \text{honest rounds } H &:= \{t : |p^t-x^t| < \eta\},
    \end{align*}
    with counts $B_T=|B|$ and $H_T=|H|=T-B_T$.

    On $B$, exactly one of $g_1$ or $g_2$ is active. On $H$, only $g_3$ is active. We chose $\eta$ small enough relative to the grid spacing  so that the intervals $(x-\eta, x+\eta)$ around distinct grid points are disjoint; this will later let us localize the $g_3$ error by context.

2. \textbf{Many big deviations force large error on $g_1$ or $g_2$.}
    The first part of our analysis shows that if the algorithm predicts dishonestly often, it pays linearly in multicalibration error.

    Let $r^t := p^t-x^t$. On big-deviation rounds, either $r^t \ge \eta$ or $r^t \le -\eta$. For each fixed prediction $v$, the expected contribution of the overshoot rounds ($r^t \ge \eta$) to the calibration bias of $g_1$ is
    $$ \mathbb{E}\big[B_T(v,g_1)\big] = \mathbb{E}\Big[\sum_{t: p^t=v} \mathbf{1}[p^t\ge x^t+\eta](p^t-y^t)\Big] = \mathbb{E}\Big[\sum_{t: p^t=v} \mathbf{1}[p^t\ge x^t+\eta]r^t\Big], $$
    since $\mathbb{E}[y^t\mid x^t,p^t]=x^t$. Whenever $p^t \ge x^t+\eta$, we have $r^t \ge \eta$, so we get expected positive bias $\geq \eta$ on those rounds; a symmetric argument holds for $g_2$. Summing over $v$, this implies
    $$ \mathbb{E}[\mathrm{MCerr}_T] \gtrsim \eta \mathbb{E}[B_T]. $$
    So, if $B_T$ is large, we are already done: either $g_1$ or $g_2$ must have large calibration error.

3. \textbf{Few big deviations force many ``honest'' rounds per context.}
    The complementary case is when the algorithm mostly stays close to honest: If $\mathbb{E}[B_T]$ is not large, then $\mathbb{E}[H_T]$ is large. Contexts cycle on a regular grid, and so each grid point $x$ appears about $T/m$ times.

    We refine the partition by context to avoid cancellations: For each grid point $x$, let $H_x$ be the honest rounds with that context, and $n_x:=|H_x|$. Similarly define $B_x$ and $b_x$ for big deviations at context $x$, with $T_x = n_x + b_x$ the total number of times $x$ appears. If $B_T$ is small, the structure of the instance forces many contexts $x$ to have a substantial number of honest rounds $n_x$; this is where we will extract noise using a martingale argument.

4. \textbf{Honest rounds accumulate uncontrollable noise on $g_3$.}
    On honest rounds for a fixed context $x$, we decompose the calibration error for $g_3$ into:
    \begin{itemize}
        \item \textbf{Noise:} $N_x := \sum_{t\in H_x} (x^t - y^t)$, a sum of centered Bernoulli deviations around $x^t$; and
        \item \textbf{Bias:} $R_x := \sum_{t\in H_x} (p^t - x^t)$, which is small in magnitude because $|p^t-x^t|<\eta$ on $H_x$.
    \end{itemize}
    We show a structural lemma (Lemma~\ref{lem:g4-context-decomp}) that the calibration error on $g_3$ satisfies
    $$ \sum_{v} |B_T(v,g_3)| \ge \sum_x |N_x| - \sum_x |R_x|. $$
By construction, we know that the bias is small on honest rounds:  $\sum_x |R_x| \le \eta H_T$ which will be negligible at our choice of $\eta$.

    The core probabilistic step is a  martingale-transform lower bound applied context-wise: on the rounds where any fixed context $x$ occurs, the noise terms $Z_t := x^t-y^t$ form a martingale difference sequence with variance bounded away from zero. Using a martingale moment argument, we show (Proposition~\ref{prop:binomial-deviation}) that whenever a nontrivial fraction of the occurrences of $x$ are honest rounds (i.e., many indices in $H_x$), the noise magnitude $\mathbb{E}[|N_x|] \gtrsim \sqrt{n_x}$ is large.

    Summing over contexts and using that the $T_x$'s are all $\Theta(T/m)$, we obtain a tradeoff: either there are many big deviations (large $B_T$), or
    $$ \mathbb{E}\Big[\sum_x |N_x|\Big] \gtrsim \sqrt{mT}, $$
    which in turn forces the calibration error on $g_3$ to be of that order up to the small bias term. Choosing parameters to optimize the tradeoff yields our optimal $\Omega(T^{2/3})$ lower bound. 

\subsubsection{Lower Bound for Prediction Independent Groups}
At a high level, our approach to proving lower bounds while restricting to
\emph{prediction-independent} groups mirrors our approach for the general case:
(1) we use an oblivious stochastic instance in which label means are revealed by the
context, so that ``honest'' predictions are feasible but suffer high calibration error because of noise; (2) we construct groups
to punish ``dishonest'' strategies that try to cancel noise by
grouping predictions; and (3) we bound the multicalibration error by decomposing the marginal calibration error into a
\emph{bias} component (controlled by group constraints) and an unavoidable \emph{noise}
component (which we show must be large).

Unlike the prediction dependent case, we cannot directly detect and constrain deviations from honesty, since our groups must be defined only through context (not prediction). Instead, we define a family of groups derived from an orthonormal basis and show that low multicalibration error on these groups requires the learner to make predictions that are close to honest in an $\ell_1$ sense. 

\paragraph{Hard instance.} Concretely, we consider contexts $x^t$ that encode the label mean. Label means $x^t$ cycle 
through an $m$-point grid in $[1/4,3/4]$ with $m=\Theta(T^{1/3})$, and outcomes are
generated as
\[
y^t \;=\; x^t + \frac{\xi^t}{4},\qquad \xi^t\in\{\pm1\}\text{ i.i.d.}
\]
We define a prediction-independent group family $G$ consisting of:
(i) the constant group $g_{\mathrm{all}}$ (enforcing marginal calibration),
(ii) $O(\log^3 T)$ Walsh half-groups $g^{\Walsh,\pm}_\ell$ on the $m$-point grid. They are subsampled from a full family of Walsh half-groups, whose differences yield signed Walsh functionals $w_\ell=g^{\Walsh,+}_\ell-g^{\Walsh,-}_\ell$ forming an orthonormal basis.

\paragraph{Proof steps.} Our starting point is a bias--noise decomposition lower bound on the calibration error of the constant group. Our multicalibration lower bound will be extracted from this bound, with the help of the other (Walsh) groups. Specifically, letting $Z_t = x_t-y_t$,
we define an $\ell_1$-truthfulness/bias term, and bucket-wise noise terms:
\[
A := \sum_{t\le T}|p_t-x_t|,
\qquad
U_v := \sum_{t\le T:p_t=v} Z_t.
\]
Then, it is easily seen that for the constant group \(g_{\mathrm{all}}\), the calibration error satisfies
\[
\mathrm{Err}_T(g_{\mathrm{all}})
\ge
\sum_{v\in V_T}|U_v| - A.
\]
The remaining proof reduces to upper-bounding the bias penalty \(A\) and lower-bounding the
noise contribution \(\sum_v |U_v|\).
We accomplish this via the following several steps:

1. \textbf{Walsh groups enforce \(\ell_1\)-truthfulness, hence control the bias.} We show that the
sign pattern \(\mathrm{sign}(v-x_i)\) on the mean grid can be exactly expanded by an auxiliary
full Walsh group family with an \(\ell_1\) coefficient mass bounded by \(O(\log m)\). We further
prove the existence of a polylogarithmic-size subsampled Walsh family that approximately
expands the sign pattern, while maintaining the logarithmically bounded total coefficient
weight. Using this subsample for the final group family, we show that if the forecaster has
small multicalibration error, then the bias penalty \(A\) is small in expectation:
\[
\mathbb{E}[A]
\le
O(\log m)\cdot \mathbb{E}[\mathrm{MCerr}_T(G)].
\]
Intuitively: systematic deviations from \(x_t\) would either be detected by marginal
calibration on \(g_{\mathrm{all}}\) or be witnessed by some Walsh functional.

2. \textbf{\(\ell_1\)-truthfulness forces many moderately-used prediction values.} It turns out that on our hard instance, a forecaster would need to spread its forecasts across many relevant values \(v\) to maintain honesty. To formalize this, define a prediction diversity/spread parameter \(N\) as follows:
\[
N := \sum_{v\in V_T}\sqrt{n_v}, \qquad \text{where } n_v := |\{t\le T:p_t=v\}|.
\]
As we show, small
\(\ell_1\)-distance to honesty \(A\) prevents the forecaster from concentrating mass on only a
few prediction values, and forces spread-out predictions in the sense that, up to constants,
\[
\mathbb{E}[N]
\gtrsim
\frac{T}{\sqrt{\mathbb{E}[A]+T/m+1}}.
\]
Since \(x_t\) traverses a grid of size \(m=\Theta(T^{1/3})\),
this implies (for small \(A\)) that \(N\) is \(\approx \sqrt{mT}\approx T^{2/3}\).

3. \textbf{Bucket noise is lower-bounded even under adaptive bucketing.} The noise terms \(Z_t=x_t-y_t\)
are i.i.d.\ Rademachers scaled by \(1/4\). We analyze an arbitrary adaptive bucketing
strategy that, at each time, chooses a bucket based on the past noise realizations. Via a
potential-function argument for the bucket sums and a decomposition of the noise random
walk into excursions away from zero, we prove an adaptive noise bucketing theorem: up to
logarithmic factors,
\[
\mathbb{E}\bigg[\sum_{v\in V_T}|U_v|\bigg]
\gtrsim
\mathbb{E}[N].
\]
In other words, no matter how the algorithm routes the noise, the total bucketed noise
magnitude must be large whenever the bucket counts are diverse.

Finally, we combine the above steps. Plugging the bias bound and the adaptive-bucketing noise
bound back into the constant-group decomposition yields
\[
\mathbb{E}[\mathrm{MCerr}_T(G)]
\ge
\widetilde{\Omega}(T^{2/3}),
\]
matching upper bounds up to polylogarithmic factors for \(|G|=O(\log^3 T)\).

\section{Model and Definitions}
\label{sec:model}

Fix a time horizon $T\in\mathbb{N}$. On each round $t=1,\dots,T$:
\begin{enumerate}
  \item A context $x^t$ in a context space $X$ is revealed.
  \item The prediction algorithm outputs a distribution $P^t$ on $[0,1]$.
  \item An outcome $y^t\in[0,1]$ is selected (by the adversary/environment).
  \item A prediction $p^t\in[0,1]$ is drawn from $P^t$.
\end{enumerate}

We allow the algorithm to be adaptive: $P^t$ can be any (possibly randomized) function of
$(x^1,y^1,p^1,\dots,x^{t-1},y^{t-1},p^{t-1},x^t)$. In all of our lower bounds, the adversary will be oblivious --- the context/label sequence is selected independently of the interaction with the prediction algorithm.

Whenever a proof needs the realized prediction $p^t$ to be predictable before $y^t$ is revealed, we condition on the
learner's full internal random seed at time $0$ and enlarge the filtration to include this seed. After this
conditioning, the sampled prediction on round $t$ is a deterministic function of the pre-outcome history
$(x^1,y^1,p^1,\dots,x^{t-1},y^{t-1},p^{t-1},x^t)$, so $p^t$ is measurable with respect to the corresponding
pre-outcome filtration. All expectations in the paper are then obtained by averaging back over the seed.

\begin{definition}[Group functions]
A \emph{group function} is a map $g:X\times[0,1]\to[0,1]$.
Given a finite set $G$ of group functions, we will measure calibration error separately on each group.
Throughout, we use $G$ to denote a finite family of group functions and write $|G|$ for its cardinality.
\end{definition}

\begin{definition}[Binary and prediction-independent groups]
A group function $g:X\times[0,1]\to[0,1]$ is \emph{binary-valued} if $g(x,v)\in\{0,1\}$ for all $(x,v)\in X\times[0,1]$.
We say $g$ is \emph{prediction-independent} if there exists a function $h:X\to[0,1]$ such that $g(x,v)=h(x)$ for all $(x,v)$.
In this case we may identify $g$ with $h$ and write $g(x)$ instead of $g(x,v)$.
\end{definition}

We sometimes refer to general group functions $g:X\times[0,1]\to[0,1]$ that may depend on the prediction value $v$ as \emph{prediction-dependent} groups. All of the lower bounds we prove in this paper use only binary groups.

\begin{definition}[Empirical Bias and multicalibration error]
Given a sequence $(x^t,p^t,y^t)_{t=1}^T$ and a group $g$, the \emph{empirical bias} at prediction value $v\in[0,1]$ is
\[
  B_T(v,g)\;=\;\sum_{t=1}^T \mathbf{1}[p^t=v]\;g(x^t,p^t)\,(p^t-y^t).
\]
Let $V_T:=\{p^t : t=1,\dots,T\}$ be the (finite) set of prediction values actually used.
The expected calibration error for group $g$ is
\[
  \mathrm{Err}_T(g)\;:=\;\sum_{v\in V_T} |B_T(v,g)|.
\]

Given a finite family $G$ of groups, the \emph{expected multicalibration error} at time $T$ is
\[
  \mathrm{MCerr}_T(G)\;:=\;\max_{g\in G}\;\mathrm{Err}_T(g)
  \;=\;\max_{g\in G}\;\sum_{v\in V_T} |B_T(v,g)|.
\]
When the group family is clear from context  we abbreviate the multicalibration error as $\mathrm{MCerr}_T$.
\end{definition}

Randomness arises both from the algorithm (if it randomizes) and from the environment.
All expectations $\mathbb{E}[\cdot]$  are taken with respect to this joint randomness.

\section{Optimal Lower Bound for the General Case}
\label{sec:prediction-dependent}

In this section, we give a lower bound instance consisting of three binary prediction-dependent groups, showing that any algorithm must obtain multicalibration error over these groups scaling as  $\Omega(T^{2/3})$. This  matches the upper bound of \cite{noarov2023high} up to log factors. 

\paragraph{Proof overview.}
The hard instance reveals contexts $x^t$ that cycle through a grid of $m = \Theta(T^{1/3})$ values, with labels drawn as $y^t \sim \mathrm{Bernoulli}(x^t)$. The key insight is that prediction-dependent groups can directly detect when the learner deviates from ``honest'' predictions $p^t = x^t$:
\begin{itemize}
    \item Groups $g_1$ and $g_2$ activate when predictions overshoot or undershoot the context by more than $\eta$. Any such ``big deviation'' incurs expected bias of at least $\eta$, so many big deviations yield large error on $g_1$ or $g_2$.
    \item Group $g_3$ activates on ``$\eta$-honest'' rounds where $|p^t - x^t| < \eta$. On these rounds, predictions are approximately honest, so calibration error is driven by the inherent noise $x^t - y^t$. A martingale argument shows this noise accumulates to $\Omega(\sqrt{mT}) = \Omega(T^{2/3})$.
\end{itemize}
Either the algorithm makes many big deviations (punished by $g_1, g_2$) or mostly honest predictions (punished by $g_3$), yielding $\Omega(T^{2/3})$ error in both cases.

\begin{theorem}[Prediction-dependent lower bound]\label{thm:main} Let $(\mathcal{D}_{T,m}, G)$ be the hard instance defined in this section. There exists a constant $c>0$ and $T_0\in\mathbb{N}$ such that for all $T\ge T_0$, and for any (possibly randomized) prediction algorithm $A$:
\[
 \mathbb{E}_{\mathcal{D}_{T,m}}[\mathrm{MCerr}_T(G)] \ \;\ge\; c\,T^{2/3}.
\]
\end{theorem}

\subsection{The Hard Instance}
We will now define both elements of our proposed hard instance: the distribution over contexts and labels, followed by the group family.

\paragraph{The Hard Distribution}
First, we define the hard distribution over contexts and labels, $\mathcal{D}_{T,m}$. For some $T\ge 1$ and $m\ge 8$, define grid points
\[
  z_j \;:=\; \frac{j}{m}, \qquad j=1,\dots,m-1.
\]
We restrict attention to the ``interior'' grid points
\[
  J \;:=\;\Big\{j\in\{1,\dots,m-1\}: z_j\in\big[\tfrac14,\tfrac34\big]\Big\},
\]
and set
\[
  X_0\;:=\;\{x_j : j\in J\}, \qquad x_j := z_j.
\]
For $m\ge 4$ we have the uniform bound
\[
  |J|
  = \bigl|\{j\in\{1,\dots,m-1\}: j/m\in[1/4,3/4]\}\bigr|
  \;\ge\; \Big\lfloor\frac{3m}{4}\Big\rfloor - \Big\lceil\frac{m}{4}\Big\rceil + 1
  \;\ge\; \frac{m-1}{2}.
\]
In particular, for all $m\ge 4$, $|J|\ge \frac{3}{8} \cdot m$.

Let $m_0 := |J|$. We will only use the contexts in $X_0$.

\begin{definition}[Hard distribution $\mathcal{D}_{T,m}$]
Fix $T,m$ and $X_0$ as above.
Define $(x^t,y^t)_{t=1}^T$ as follows:
\begin{itemize}
  \item Contexts: $x^t$ cycles through $X_0$ in round-robin order.
        Formally, fix any enumeration $(x^{(1)},\dots,x^{(m_0)})$ of $X_0$, and set
        \[
          x^t := x^{(k)}, \quad\text{where } k \equiv t \pmod{m_0},\ k\in\{1,\dots,m_0\}.
        \]
        Thus each $x^{(k)}$ appears either $\lfloor T/m_0 \rfloor$ or $\lceil T/m_0 \rceil$ times.
  \item Labels: given the context $(x^t)$, draw $y^1,\dots,y^T$ independently with
        \[
        	  y^t \sim \mathrm{Bernoulli}(x^t)
        \]
\end{itemize}
We denote by $\mathcal{D}_{T,m}$ the joint distribution of $(x^t,y^t)_{t=1}^T$ constructed above. For our purposes, the order of the contexts will not matter: the important properties of this distribution are that each context is approximately equally frequent, and that it encodes the label mean at each round. 
\end{definition}
We observe a useful property about this distribution. 
For each $t$, define
  $Z_t := x^t - y^t$, the residuals of ``honest predictions''. 
Then $(Z_t)_{t=1}^T$ are independent, mean zero, and have nontrivial variance because $x^t\in [1/4,3/4]$ for all $t$:
\[
  \mathbb{E}[Z_t\mid x^t] = 0, \qquad
  \mathrm{Var}(Z_t\mid x^t) = x^t(1-x^t) \in \big[\tfrac{3}{16},\tfrac{1}{4}\big].
\]

\paragraph{The Hard Group Family}
We now define the family of disjoint groups $G = \{g_1,g_2,g_3\}$ used for the lower bound. The groups are designed to create a dilemma for the forecaster: deviating from honest predictions is detected by $g_1$ and $g_2$, while staying honest exposes the forecaster to noise accumulation on $g_3$.

\noindent
Let $\eta > 0$ be a threshold (it will later be set to $\delta \sqrt{\frac{m}{T}}$ for a carefully selected $\delta$). The groups are:
\begin{align*}
  g_1(x,v) := \mathbf{1}[v \ge x + \eta], \qquad
  g_2(x,v) := \mathbf{1}[v \le x - \eta], \qquad
  g_3(x,v) := \mathbf{1}[|v-x| < \eta].
\end{align*}

Thus:
\begin{itemize}
  \item $g_1$ is active on rounds where the prediction overshoots the context by at least $\eta$;
  \item $g_2$ is active on rounds where the prediction undershoots by at least $\eta$;
  \item $g_3$ is active on rounds where the prediction is $\eta$-close to the context.
\end{itemize}

\subsection{Probabilistic Tools}

The main technical tool we need is a lower bound on the deviation of a ``filtered'' martingale. When analyzing the $\eta$-honest rounds (group $g_3$), we must show that the noise terms $Z_t = x^t - y^t$ accumulate to large magnitude even when summed only over an adaptively-chosen subset of rounds. The key challenge is that the forecaster's predictions---and hence which rounds are $\eta$-honest---can depend on past noise realizations, potentially allowing cancellations.

Proposition~\ref{prop:binomial-deviation} below shows this cannot help much: as long as at least a constant fraction of rounds are included (in expectation), the filtered sum still has $\Omega(\sqrt{L})$ expected magnitude.

We now record some basic properties of a martingale difference sequence $(Z_t)$ that arise from the Bernoulli environment we use in our lower bound. 
Let $(x^t)_{t=1}^T \subset [1/4,3/4]$, and given $(x^t)$ let $y^1,\dots,y^T$ be independent with $y^t\sim\mathrm{Bernoulli}(x^t)$ for each $t$.
 $Z_t := x^t-y^t$. Consequently,
\[
  x^t-y^t \in \Big[-\tfrac{3}{4},-\tfrac{1}{4}\Big]\cup\Big[\tfrac{1}{4},\tfrac{3}{4}\Big]
  \qquad\text{and hence}\qquad
  \tfrac{1}{4} \,\le\, |Z_t| \,\le\, \tfrac{3}{4}
\]
for every $t$.

The contexts $(x^t)$ are deterministic (fixed independently of the history, as the adversary is non-adaptive/oblivious), and the labels
$y^1,\dots,y^T$ are independent with $y^t\sim\mathrm{Bernoulli}(x^t)$.
Let
\[
  \mathcal{F}_t := \sigma(y^1,\dots,y^t),\qquad t=0,1,\dots,T,
\]
with the convention that $\mathcal{F}_0$ is the trivial $\sigma$-algebra.
Then $Z_t = x^t-y^t$ is $\mathcal{F}_t$-measurable and
\[
  \mathbb{E}[Z_t\mid \mathcal{F}_{t-1}]
  = x^t - \mathbb{E}[y^t\mid \mathcal{F}_{t-1}]
  = x^t - \mathbb{E}[y^t]
  = x^t - x^t
  = 0,
\]
since $y^t$ is independent of $(y^1,\dots,y^{t-1})$ and has mean $x^t$.
Thus $(Z_t,\mathcal{F}_t)$ is a martingale difference sequence.
Moreover, for a fixed  prediction algorithm, each prediction $p^t$ is a measurable
function of $(x^1,\dots,x^t,y^1,\dots,y^{t-1})$.
 Because the contexts $(x^t)$ are deterministic under $\mathcal{D}_{T,m}$, this implies
  $p^t$ is $\mathcal{F}_{t-1}$-measurable for every $t$.

 Multicalibration error adds up the magnitude of \emph{empirical bias} the algorithm has obtained over various subsequences of the data, and those subsequences can be defined by the predictions of the algorithm itself, which in turn can depend on the history of the sequence in arbitrary ways. To reason about this we will use the following deviation bound for a martingale transform $N = \sum_{t=1}^L I_t Z_t$,
 where $(Z_t)$ is a martingale difference sequence and the predictable indicators $(I_t)$ select a
 dense subset of times (in expectation at least $\alpha L$).

\begin{proposition}[Dense martingale transform deviation]\label{prop:binomial-deviation}
Fix constants $0<\sigma\le 1$ and $\alpha\in(0,1]$.
Let $(Z_t)_{t=1}^L$ be a sequence of real-valued random variables adapted to a
filtration $(\mathcal{F}_t)_{t=0}^L$ such that
\begin{equation}\label{eq:Z-assumptions-dense}
  \mathbb{E}[Z_t\mid\mathcal{F}_{t-1}] = 0, \qquad
  \mathbb{E}[Z_t^2\mid\mathcal{F}_{t-1}] \ge \sigma^2, \qquad
  |Z_t|\le 1
  \quad\text{almost surely for all }t=1,\dots,L.
\end{equation}
Let $(I_t)_{t=1}^L$ be any predictable $\{0,1\}$-valued sequence, i.e.\ $I_t$
is $\mathcal{F}_{t-1}$-measurable for each $t$, and define
\[
  N := \sum_{t=1}^L I_t Z_t,
  \qquad
  n := \sum_{t=1}^L I_t.
\]
If $\mathbb{E}[n] \ge \alpha L$, then there exists a constant
$c_{\sigma,\alpha}>0$ (depending only on $\sigma$ and $\alpha$) such that
\[
  \mathbb{E}|N| \;\ge\; c_{\sigma,\alpha}\,\sqrt{L}.
\]
\end{proposition}

\begin{proof}
Define the martingale $(M_t)_{t=0}^L$ by $M_0=0$ and
$
  M_t := \sum_{s=1}^t I_s Z_s, \, t=1,\dots,L.
$

Since $I_s$ is $\mathcal{F}_{s-1}$-measurable and~\eqref{eq:Z-assumptions-dense}
gives $\mathbb{E}[Z_s\mid\mathcal{F}_{s-1}]=0$, we have
\[
  \mathbb{E}[M_t\mid\mathcal{F}_{t-1}]
  = M_{t-1} + \mathbb{E}[I_t Z_t\mid\mathcal{F}_{t-1}]
  = M_{t-1} + I_t\,\mathbb{E}[Z_t\mid\mathcal{F}_{t-1}]
  = M_{t-1},
\]
so $(M_t,\mathcal{F}_t)$ is a martingale. Its increments are $D_t = I_t Z_t$,
and its predictable quadratic variation is
\[
  \langle M\rangle_L
  := \sum_{t=1}^L \mathbb{E}[D_t^2\mid\mathcal{F}_{t-1}]
  = \sum_{t=1}^L I_t\,\mathbb{E}[Z_t^2\mid\mathcal{F}_{t-1}].
\]
By~\eqref{eq:Z-assumptions-dense}, whenever $I_t=1$ we have
$\mathbb{E}[Z_t^2\mid\mathcal{F}_{t-1}] \ge \sigma^2$, while if $I_t=0$
the corresponding term is zero. Therefore, pathwise
\[
  \langle M\rangle_L \;\ge\; \sigma^2 \sum_{t=1}^L I_t
  = \sigma^2 n.
\]
Taking expectations yields
the  identity $\mathbb{E}[M_L^2] = \mathbb{E}[\langle M\rangle_L]$, since
the cross terms in $M_L^2 = (\sum_{t=1}^L D_t)^2$ vanish by $\mathbb{E}[D_t\mid\mathcal{F}_{t-1}]=0$.
Thus  using the assumption $\mathbb{E}[n]\ge \alpha L$,
\begin{equation}\label{eq:EN2-lower}
  \mathbb{E}[M_L^2] = \mathbb{E}[\langle M\rangle_L]
  \;\ge\; \sigma^2 \mathbb{E}[n]
  \;\ge\; \sigma^2 \alpha L.
\end{equation}

On the other hand, since $|Z_t|\le 1$ we have $|D_t| \le |I_t||Z_t| \le 1$ almost surely.
Also $D_t^2 \le I_t$, so $\langle M\rangle_L \le \sum_{t=1}^L I_t = n \le L$ almost surely.
Since $|M_L| \le \sup_{0\le t\le L} |M_t|$, we have
\[
  \mathbb{E}[M_L^4] \;\le\; \mathbb{E}\big[\sup_{0\le t\le L} |M_t|^4\big].
\]

\noindent
We will now make use of the Burkholder--Rosenthal inequality.

\begin{lemma}[Burkholder--Rosenthal Inequality~\citep{burkholder1973distribution}]
Let $(M_t)_{t=0}^T$ be a martingale with increments $D_t = M_t - M_{t-1}$ and predictable quadratic variation $\langle M \rangle_T = \sum_{t=1}^T \mathbb{E}[D_t^2 \mid \mathcal{F}_{t-1}]$. For any $p \ge 2$, there exists a constant $C_p$ depending only on $p$ such that:
\[
\mathbb{E}\left[ \sup_{0 \le t \le T} |M_t|^p \right] \le C_p \left( \mathbb{E}\left[ \langle M \rangle_T^{p/2} \right] + \mathbb{E}\left[ \sup_{1 \le t \le T} |D_t|^p \right] \right).
\]
\end{lemma}

\noindent
Applying the Burkholder--Rosenthal inequality to the above with $p=4$ gives, assuming $L\ge 1$, that
\begin{equation}\label{eq:EN4-upper}
  \mathbb{E}[M_L^4]
  \;\le\; C_4\,\Big(\mathbb{E}\left[\langle M\rangle_L^2\right] + \mathbb{E}\big[\sup_{1\le t\le L} |D_t|^4\big]\Big)
  \;\le\; C_4\,(L^2 + 1)
  \;\le\; 2C_4 L^2.
\end{equation}

To finish the proof, define the random variable $X := M_L^2$.
From \eqref{eq:EN2-lower} we have
$\mathbb{E}[X] = \mathbb{E}[M_L^2]\ge \sigma^2 \alpha L$, and from
\eqref{eq:EN4-upper} we have
$\mathbb{E}[X^2] = \mathbb{E}[M_L^4]\le 2 C_4 L^2$. Now, consider the event $\{X\ge \tfrac12\,\mathbb{E}[X]\}$. On this event, we have:
\[
  |M_L| = \sqrt{X}
  \;\ge\; \sqrt{\tfrac12\,\mathbb{E}[X]}
  \;\ge\; \sqrt{\tfrac12\,\sigma^2\alpha L}
  = \sigma\sqrt{\tfrac{\alpha}{2}}\,\sqrt{L}.
\]

\noindent
Recall the Paley-Zygmund inequality: for any r.v.\ $\widetilde{X} \ge 0$ with $0<\mathbb{E}[\widetilde{X}^2]<\infty$, it holds for any $\theta\in(0,1)$ that $ \mathbb{P}\big(\widetilde{X} \ge \theta\,\mathbb{E}[\widetilde{X}]\big)
  \;\ge\; (1-\theta)^2\,\frac{\mathbb{E}[\widetilde{X}]^2}{\mathbb{E}[\widetilde{X}^2]}.$
Applying this to $X$ with $\theta=1/2$, we get
\[
  \mathbb{P}\Big(X \ge \tfrac12\,\mathbb{E}[X]\Big)
  \;\ge\; (1-\tfrac12)^2\,\frac{\mathbb{E}[X]^2}{\mathbb{E}[X^2]}
  \;\ge\; \frac{1}{4}\,\frac{(\sigma^2\alpha L)^2}{2 C_4 L^2}
  =: p_0 > 0,
\]
where $p_0$ depends only on $\sigma$ and $\alpha$.
Hence, the proposition holds for $c_{\sigma,\alpha} := \sigma p_0 \sqrt{\alpha/2}$, since:
\[
  \mathbb{E}|N|
  = \mathbb{E}|M_L|
  \;\ge\; \sigma\sqrt{\tfrac{\alpha}{2}}\sqrt{L}\cdot
          \mathbb{P}\Big(X \ge \tfrac12\,\mathbb{E}[X]\Big)
  \;\ge\; \sigma\sqrt{\tfrac{\alpha}{2}}\sqrt{L}\cdot p_0. \qedhere
\]
\end{proof}

\subsection{Proof of Theorem~\ref{thm:main}}
We are now ready to prove our lower bound for prediction-dependent groups, using the instance $(\mathcal{D}_{T,m},G)$ defined above.

Given an algorithm and a realization of $(x^t,p^t)_{t=1}^T$, we will partition the rounds as follows.

\begin{definition}[Big deviations and $\eta$-honest rounds]
Define the index sets
\begin{align*}
  B &:= \{t : |p^t - x^t| \ge \eta\}, \\
  H &:= \{t : |p^t - x^t| < \eta\},
\end{align*}
and denote their sizes by
\[
  B_T := |B|, \qquad H_T := |H|.
\]
We refer to $t\in B$ as \emph{big-deviation rounds} and $t\in H$ as 
\emph{$\eta$-honest rounds}.
\end{definition}

Note that $B_T+H_T=T$ holds pathwise.

\medskip

 The proof of Theorem \ref{thm:main} reduces to two main components which we will combine at the end:

  \begin{enumerate}
  \item Big-deviation rounds incur large error under $g_1$ or $g_2$ (Section~\ref{sec:big-lies}).
  \item Many $\eta$-honest rounds incur large error under $g_3$ (Section~\ref{sec:g4}).
  \end{enumerate}

\subsubsection{Big Deviations are Punished by \texorpdfstring{$g_1$}{g1} and \texorpdfstring{$g_2$}{g2}}\label{sec:big-lies}

In this section we prove that big deviations (rounds $t$ with $|p^t-x^t|\ge\eta$) necessarily incur large expected error on $g_1$ or $g_2$.

We first note that since our lower bound instance is fixed and non-adaptive, it suffices to consider deterministic predictors.

\begin{lemma}[Reduction to deterministic predictors]\label{lem:deterministic}
Fix $T,m$ and the distribution $\mathcal{D}_{T,m}$.
Let $A$ be any (possibly randomized) prediction algorithm.
Then there exists a deterministic algorithm $A'$ such that
\[
  \mathbb{E}_{\mathcal{D}_{T,m}}[\mathrm{MCerr}_T(A')]
  \;\le\; \mathbb{E}_{\mathcal{D}_{T,m},A}[\mathrm{MCerr}_T],
\]
where the first expectation is over $\mathcal{D}_{T,m}$ only (since $A'$ is
deterministic), and the second expectation is over both the randomness of $A$
and $\mathcal{D}_{T,m}$.
\end{lemma}

\begin{proof}
We model all of the algorithm's internal randomization (including any sampling
of $p^t$ from an internal distribution $P^t$) by an independent random seed $R$.
For each outcome of $R$, the resulting algorithm $A_R$ produces predictions
$p^t$ as deterministic functions of the history
$(x^1,y^1,\dots,x^{t-1},y^{t-1},x^t)$. The seed $R$ is independent of
$(x^t,y^t)_{t=1}^T$ (as the lower bound instance is fixed and non-adaptive), so we can view $(x^t,y^t,R)$ as the environment's sample
space. Then
\[
  \mathbb{E}_{\mathcal{D}_{T,m},A}[\mathrm{MCerr}_T]
  = \mathbb{E}_R\big[\mathbb{E}_{\mathcal{D}_{T,m}}[\mathrm{MCerr}_T(A_R)]\big]
\]
is an average of the quantities $\mathbb{E}_{\mathcal{D}_{T,m}}[\mathrm{MCerr}_T(A_R)]$ over all outcomes $R$.
Therefore there exists at least one outcome $R'$ such that
\[
  \mathbb{E}_{\mathcal{D}_{T,m}}[\mathrm{MCerr}_T(A_{R'})]
  \;\le\; \mathbb{E}_{\mathcal{D}_{T,m},A}[\mathrm{MCerr}_T].
\]
Taking $A' := A_{R'}$ yields the claim.
\end{proof}

Thus it suffices to fix an arbitrary deterministic algorithm and analyze its error.

\medskip

Fix a deterministic prediction algorithm and $\mathcal{D}_{T,m}$.
For each $t$, $p^t$ is a deterministic function of $(x^1,y^1,\dots,x^{t-1},y^{t-1},x^t)$.
Let $B_T$ be the (random) number of big-deviation rounds:
\[
  B_T := |\{t : |p^t - x^t| \ge \eta\}| = |B|.
\]

First we show that the multicalibration error of any algorithm grows linearly with the number of big deviations, witnessed by either $g_1$ or $g_2$.

\begin{lemma}\label{lem:big-lies}
For any deterministic prediction algorithm and any $\eta>0$,
under $\mathcal{D}_{T,m}$ we have
\[
  \mathbb{E}[\mathrm{MCerr}_T]
  \;\ge\; \frac{\eta}{2}\,\mathbb{E}[B_T].
\]
\end{lemma}

\begin{proof}
Write $r^t := p^t - x^t$ to denote the round $t$ deviation from ``honest'' prediction. Define the indicators
\[
  I_t^+ := \mathbf{1}[p^t \ge x^t + \eta],\qquad
  I_t^- := \mathbf{1}[p^t \le x^t - \eta].
\]
to indicate whether round $t$ represented a big deviation in either the positive or negative direction. 
Then $|p^t-x^t|\ge\eta$ implies $I_t^+ + I_t^- = 1$.
Hence
\[
  B_T = \sum_{t=1}^T (I_t^+ + I_t^-).
\]

Consider group $g_1$. For a fixed $v\in[0,1]$,
\[
  B_T(v,g_1)
  = \sum_{t=1}^T \mathbf{1}[p^t=v]\;g_1(x^t,p^t)\,(p^t-y^t)
  = \sum_{t=1}^T \mathbf{1}[p^t=v]\;I_t^+\,(p^t-y^t).
\]

For each $t$, let $\mathcal{F}_t := \sigma(x^1,y^1,\dots,x^{t-1},y^{t-1},x^t,p^t)$ be the history
available just before $y^t$ is revealed.
Then $\mathbf{1}[p^t=v]$ and $I_t^+$ are $\mathcal{F}_t$-measurable, and under $\mathcal{D}_{T,m}$
the label $y^t$ satisfies
\[
  \mathbb{E}[y^t \mid \mathcal{F}_t] = \mathbb{E}[y^t \mid x^t] = x^t,
\]
because $y^t$ is drawn independently of the past given $x^t$.
Thus
\[
  \mathbb{E}[p^t-y^t \mid \mathcal{F}_t] = p^t - x^t = r^t.
\]
Using the tower property and linearity of expectation,
\begin{align*}
  \mathbb{E}[B_T(v,g_1)]
  &= \mathbb{E}\Big[\sum_{t=1}^T \mathbf{1}[p^t=v]\,I_t^+\,(p^t-y^t)\Big] \\
  &= \sum_{t=1}^T \mathbb{E}\big[\mathbf{1}[p^t=v]\,I_t^+\,\mathbb{E}[p^t-y^t \mid \mathcal{F}_t]\big] \\
  &= \sum_{t=1}^T \mathbb{E}\big[\mathbf{1}[p^t=v]\,I_t^+\,r^t\big].
\end{align*}
Whenever $I_t^+=1$, we have $r^t=p^t-x^t\ge\eta$, so
\[
  \mathbb{E}[B_T(v,g_1)]
  \;\ge\; \eta\,\mathbb{E}\Big[\sum_{t=1}^T \mathbf{1}[p^t=v]\,I_t^+\Big].
\]
In particular, $\mathbb{E}[B_T(v,g_1)]\ge 0$, so Jensen's inequality gives
\[
  \mathbb{E}\bigl[|B_T(v,g_1)|\bigr]
  \;\ge\; \big|\mathbb{E}[B_T(v,g_1)]\big|
  = \mathbb{E}[B_T(v,g_1)].
\]
Summing over $v$ and using $\sum_{v\in V_T} \mathbf{1}[p^t=v]=1$,
\[
  \sum_{v\in V_T} B_T(v,g_1)
  = \sum_{t=1}^T I_t^+\,(p^t-y^t),
\]
so Jensen's inequality gives
\begin{align*}
  \mathbb{E}\Big[\sum_{v\in V_T} |B_T(v,g_1)|\Big]
  &\ge \mathbb{E}\Big[\Big|\sum_{v\in V_T} B_T(v,g_1)\Big|\Big] \\
  &\ge \Big|\mathbb{E}\Big[\sum_{v\in V_T} B_T(v,g_1)\Big]\Big| \\
  &= \Big|\sum_{t=1}^T \mathbb{E}\big[I_t^+\,(p^t-y^t)\big]\Big| \\
  &= \sum_{t=1}^T \mathbb{E}[I_t^+\,r^t] \\
  &\ge \eta\,\mathbb{E}\Big[\sum_{t=1}^T I_t^+\Big].
\end{align*}

Equivalently, for $g_2$, we have that 

\begin{align*}
  \mathbb{E}[B_T(v,g_2)]
  &= \sum_{t=1}^T \mathbb{E}\big[\mathbf{1}[p^t=v]\,I_t^- r^t\big].
\end{align*}
Whenever $I_t^-=1$, we have $r^t\le -\eta$, so
\[
  \mathbb{E}[B_T(v,g_2)]
  \;\le\; -\eta\,\mathbb{E}\Big[\sum_{t=1}^T \mathbf{1}[p^t=v]\,I_t^-\Big].
\]
In particular, $\mathbb{E}[B_T(v,g_2)]\le 0$, so
\[
  \mathbb{E}\bigl[|B_T(v,g_2)|\bigr]
  \;\ge\; \big|\mathbb{E}[B_T(v,g_2)]\big|
  = -\mathbb{E}[B_T(v,g_2)].
\]
Summing over $v$ and using $\sum_{v\in V_T} \mathbf{1}[p^t=v]=1$,
\[
  \sum_{v\in V_T} B_T(v,g_2)
  = \sum_{t=1}^T I_t^-\,(p^t-y^t),
\]
so again by Jensen's inequality,
\begin{align*}
  \mathbb{E}\Big[\sum_{v\in V_T} |B_T(v,g_2)|\Big]
  &\ge \mathbb{E}\Big[\Big|\sum_{v\in V_T} B_T(v,g_2)\Big|\Big] \\
  &\ge \Big|\mathbb{E}\Big[\sum_{v\in V_T} B_T(v,g_2)\Big]\Big| \\
  &= \Big|\sum_{t=1}^T \mathbb{E}\big[I_t^-\,(p^t-y^t)\big]\Big| \\
  &= -\sum_{t=1}^T \mathbb{E}[I_t^-\,r^t] \\
  &\ge \eta\,\mathbb{E}\Big[\sum_{t=1}^T I_t^-\Big].
\end{align*}

Combine the two:
\[
  \mathbb{E}\Big[\sum_{v\in V_T} |B_T(v,g_1)| + \sum_{v\in V_T} |B_T(v,g_2)|\Big]
  \;\ge\;\eta\,\mathbb{E}\Big[\sum_{t=1}^T (I_t^+ + I_t^-)\Big]
  = \eta\,\mathbb{E}[B_T].
\]

Since
\[
  \mathrm{MCerr}_T \;\ge\; \max\Big\{ \sum_{v} |B_T(v,g_1)|,\ \sum_{v} |B_T(v,g_2)|\Big\}
  \;\ge\; \frac12\Big(\sum_v |B_T(v,g_1)| + \sum_v |B_T(v,g_2)|\Big),
\]
we obtain
\[
  \mathbb{E}[\mathrm{MCerr}_T]
  \;\ge\; \frac12\,
  \mathbb{E}\Big[\sum_{v} |B_T(v,g_1)| + \sum_{v} |B_T(v,g_2)|\Big]
  \;\ge\; \frac{\eta}{2}\,\mathbb{E}[B_T],
\]
as claimed.
\end{proof}

\subsubsection{Many \texorpdfstring{$\eta$}{eta}-Honest Rounds are Punished by \texorpdfstring{$g_3$}{g3}}\label{sec:g4}

We now show that if the algorithm does not make many big deviations (i.e., if
the number $B_T$ of big-deviation rounds is not too large), then it must
accumulate large error on group $g_3$, which records the calibration error in the $\eta$-honest rounds.

Again fix a deterministic prediction algorithm and $\mathcal{D}_{T,m}$.
Recall that $H = \{t : |p^t-x^t|<\eta\}$ and $B = \{t : |p^t-x^t|\ge\eta\}$, with
$H_T = |H|$ and $B_T = |B| = T-H_T$.

For each context $x\in X_0$, we denote the rounds in which the context was $x$ and the prediction was $\eta$-honest (respectively a big deviation) as:
\[
  H_x := \{t\in H : x^t = x\}, \qquad n_x := |H_x|,
\]
and
\[
  B_x := \{t : x^t=x,\ |p^t-x^t|\ge \eta\}, \qquad b_x := |B_x|.
\]
Thus $T_x := n_x + b_x$ is the (deterministic) number of times the context $x$
appears in the sequence $(x^t)_{t=1}^T$. Per the definition of
$\mathcal{D}_{T,m}$, each $T_x$ is either $\lfloor T/m_0\rfloor$ or
$\lceil T/m_0\rceil$, so there exist constants $c_{\mathrm{occ}},C_{\mathrm{occ}}>0$
such that
\begin{equation}\label{eq:Tx-bounds}
  c_{\mathrm{occ}}\,\frac{T}{m_0}
  \;\le\; T_x \;\le\; C_{\mathrm{occ}}\,\frac{T}{m_0}
  \qquad\text{for all }x\in X_0.
\end{equation}
We also have
\[
  B_T = \sum_{x\in X_0} b_x,
  \qquad
  H_T = \sum_{x\in X_0} n_x,
  \qquad
  \sum_{x\in X_0} T_x = T.
\]

On $\eta$-honest rounds for context $x$, we define the noise and drift
contributions
\[
  N_x := \sum_{t\in H_x} Z_t = \sum_{t\in H_x} (x^t-y^t),
  \qquad
  R_x := \sum_{t\in H_x} (p^t-x^t).
\]
Note that $|p^t-x^t|<\eta$ on $H_x$, so $|R_x|\le \eta n_x$.

For group $g_3$, the contribution of context $x$ to the bias is
\[
  S_x := \sum_{t\in H_x} (p^t-y^t) = N_x + R_x.
\]

We next show that the calibration error over group $g_3$ is always at least the summed magnitude of the ``honest'' noise terms (summed over contexts $x$) minus the summed magnitude of the ``honest'' drift terms. 

\begin{lemma}\label{lem:g4-context-decomp}
For any realization, if $\eta \leq 1/(2m)$:
\[
  \sum_{v\in V_T} |B_T(v,g_3)|
  \;\ge\; \sum_{x\in X_0} |N_x| \;-\; \sum_{x\in X_0} |R_x|.
\]
\end{lemma}
The condition $\eta \le 1/(2m)$ ensures that the $\eta$-neighborhoods around different grid points $x \in X_0$ are disjoint, so each prediction value $v$ can be $\eta$-close to at most one context. This allows us to decompose the $g_3$ error cleanly by context.

\begin{proof}
For each context $x\in X_0$ and prediction value $v$, let
\[
  H_{x,v} := \{t\in H_x : p^t = v\},
  \qquad
  S_{x,v} := \sum_{t\in H_{x,v}} (p^t-y^t).
\]
Then $H_x = \bigsqcup_v H_{x,v}$ and
\[
  S_x = \sum_{t\in H_x} (p^t-y^t) = \sum_v S_{x,v} = N_x+R_x.
\]

\noindent
By the definition of $g_3$,
\[
  B_T(v,g_3) = \sum_{t=1}^T \mathbf{1}[p^t=v]\,g_3(x^t,p^t)\,(p^t-y^t)
  = \sum_{x\in X_0} S_{x,v}.
\]
If $\eta\le 1/(2m)$ then for each prediction value $v$
there is at most one context $x\in X_0$ such that $|v-x|<\eta$. Hence, for each $v$,
at most one of the sets $H_{x,v}$ is nonempty, and therefore at most one $S_{x,v}$
is nonzero. Thus $B_T(v,g_3)=S_{x,v}$ for that $x$, and we obtain
\[
  \sum_{v\in V_T} |B_T(v,g_3)|
  = \sum_{v\in V_T} \sum_{x\in X_0} |S_{x,v}|
  = \sum_{x\in X_0} \sum_{v\in V_T} |S_{x,v}|.
\]
For each fixed $x$, by the triangle inequality,
\[
  \sum_{v\in V_T} |S_{x,v}|
  \;\ge\; \Big|\sum_{v\in V_T} S_{x,v}\Big|
  = |S_x|
  = |N_x+R_x|
  \;\ge\; |N_x| - |R_x|.
\]
Summing over $x\in X_0$ yields the claimed bound.
\end{proof}

Since $|R_x|\le \eta n_x$ and $\sum_x n_x = H\le T$, we have
\begin{equation}
  \sum_{x\in X_0} |R_x|
  \;\le\; \eta \sum_{x\in X_0} n_x
  = \eta H
  \;\le\; \eta T.
\end{equation}
Taking expectations in Lemma~\ref{lem:g4-context-decomp} and using
this inequality gives
\begin{equation}\label{eq:g4-main-reduction}
  \mathbb{E}\Big[\sum_{v\in V_T} |B_T(v,g_3)|\Big]
  \;\ge\; \mathbb{E}\Big[\sum_{x\in X_0} |N_x|\Big] - \eta T.
\end{equation}
As such, it remains to obtain a lower bound on $\mathbb{E}\sum_x |N_x|$.

\subsection{Context-wise Tradeoff between Big Deviations and Noise}

We now establish the key tradeoff: for each context $x$, either the algorithm makes many big deviations (contributing to $g_1/g_2$ error) or it makes mostly $\eta$-honest predictions (contributing noise to $g_3$ error). By aggregating over contexts, we will show that at least one of these error sources must be large.

For each $x\in X_0$, the sequence of labels $\{y^t : x^t=x\}$ consists of
$T_x$ independent draws from $\mathrm{Bernoulli}(x)$, and therefore
$Z_t = x^t-y^t$ takes values in $[-3/4,-1/4]\cup[1/4,3/4]$ with
\[
  \mathbb{E}[Z_t\mid x^t=x] = 0,
  \qquad
  \mathbb{E}[Z_t^2\mid x^t=x] = x(1-x) \in \big[\tfrac{3}{16},\tfrac{1}{4}\big].
\]
Fix a context $x\in X_0$ and focus only on the subsequence of rounds with $x^t=x$.
We will reindex these rounds in their own ``local time'' and define filtrations that,
at each such step, contain exactly the information revealed after the prediction on that
round but before its label is drawn.

For each $t$, let
\[
  \mathcal{H}_t := \sigma\big(x^1,\dots,x^t, y^1,\dots,y^{t-1}, p^1,\dots,p^t\big)
\]
be the $\sigma$-field generated by the history just after the prediction $p^t$
is chosen and just before $y^t$ is revealed.
Under $\mathcal{D}_{T,m}$ the label $y^t$ is independent of $\mathcal{H}_t$
with mean $x^t$, so for $Z_t=x^t-y^t$ we have
\begin{equation}\label{eq:Z-Ht}
  \mathbb{E}[Z_t\mid \mathcal{H}_t] = 0,
  \qquad
  \mathbb{E}[Z_t^2\mid \mathcal{H}_t] = x^t(1-x^t) \in \big[\tfrac{3}{16},\tfrac{1}{4}\big].
\end{equation}

For each $x\in X_0$, let $t_1<\dots<t_{T_x}$ be the times with $x^{t_i}=x$.
We define a local filtration
\[
  \mathcal{F}^{(x)}_0 := \mathcal{H}_{t_1},
  \qquad
  \mathcal{F}^{(x)}_i := \mathcal{H}_{t_{i+1}} \text{ for } i=1,\dots,T_x-1,
\]
and set $\mathcal{F}^{(x)}_{T_x} := \sigma(\mathcal{H}_{t_{T_x}}, y^{t_{T_x}})$.
Clearly $(\mathcal{F}^{(x)}_i)_{i=0}^{T_x}$ is a filtration and, for
each $i=1,\dots,T_x$, the random variable
\[
  Z^{(x)}_i := x-y^{t_i}
\]
is $\mathcal{F}^{(x)}_i$-measurable. Moreover, by~\eqref{eq:Z-Ht} and the fact
that $\mathcal{F}^{(x)}_{i-1}$ contains $\mathcal{H}_{t_i}$, we have
\[
  \mathbb{E}[Z^{(x)}_i\mid \mathcal{F}^{(x)}_{i-1}] = 0,
  \qquad
  \mathbb{E}\bigl[(Z^{(x)}_i)^2\mid \mathcal{F}^{(x)}_{i-1}\bigr]
  \in \big[\tfrac{3}{16},\tfrac{1}{4}\big],
  \qquad
  |Z^{(x)}_i|\le 1.
\]

Finally, define
\[
  I^{(x)}_i := \mathbf{1}[t_i\in H_x], \qquad i=1,\dots,T_x.
\]
Membership $t_i\in H_x$ depends only on the context $x^{t_i}$ and the
prediction $p^{t_i}$, both of which are $\mathcal{H}_{t_i}$-measurable; hence
$I^{(x)}_i$ is $\mathcal{F}^{(x)}_{i-1}$-measurable.
Thus the sequence $(Z^{(x)}_i,\mathcal{F}^{(x)}_i)$ together with the
predictable indicators $I^{(x)}_i$ satisfies the assumptions
\eqref{eq:Z-assumptions-dense} of Proposition~\ref{prop:binomial-deviation}
with $\sigma^2=3/16$ and $L=T_x$.
Moreover,
\[
  N_x = \sum_{t\in H_x} Z_t
  = \sum_{i=1}^{T_x} I^{(x)}_i Z^{(x)}_i,
  \qquad
  n_x = \sum_{i=1}^{T_x} I^{(x)}_i.
\]

We define
\[
  \overline{B}_x := \mathbb{E}[b_x].
\]
We know that $\sum_x b_x = B_T$ pathwise, hence
$\sum_x \overline{B}_x = \mathbb{E}[B_T]$.

Intuitively, if the total expected number of big-deviation rounds $\mathbb{E}[B_T]$ is small, then big deviations cannot be spread across too many contexts---they must concentrate on a small subset. The remaining ``dense'' contexts have mostly $\eta$-honest predictions, which means the noise terms $N_x$ on these contexts are sums over a dense (constant-fraction) subset of the $T_x$ rounds. This is precisely the setting where Proposition~\ref{prop:binomial-deviation} applies, giving $\mathbb{E}|N_x| = \Omega(\sqrt{T_x})$ for each dense context.

\begin{lemma}\label{lem:dense-contexts} 
Suppose $\mathbb{E}[B_T] \le \frac{T}{4}$. Then there exists a subset
$D\subseteq X_0$ of contexts with
\begin{equation}\label{eq:D-Tx-mass}
  \sum_{x\in D} T_x \;\ge\; \frac{T}{2},
\end{equation}
such that for every $x\in D$ we have
\begin{equation}\label{eq:dense-Ex}
  \mathbb{E}[n_x] \;\ge\; \frac{T_x}{2}.
\end{equation}
\end{lemma}

\begin{proof}
Define the set of ``sparse'' contexts
\[
  S := \Big\{x\in X_0 : \overline{B}_x > \frac{1}{2} T_x\Big\},
  \qquad
  D := X_0\setminus S.
\]
Then
\[
  \sum_{x\in S} \overline{B}_x > \frac{1}{2} \sum_{x\in S} T_x.
\]
On the other hand, $\sum_x \overline{B}_x = \mathbb{E}[B_T]\le \frac{T}{4}
= \frac{1}{4}\sum_x T_x$, so
\[
  \frac{1}{2} \sum_{x\in S} T_x
  < \sum_{x\in S} \overline{B}_x
  \le \sum_x \overline{B}_x
  \le \frac{1}{4}\sum_x T_x.
\]
Multiplying by $2$ yields
\[
  \sum_{x\in S} T_x < \frac{1}{2}\sum_x T_x = \frac{T}{2},
\]
and therefore \eqref{eq:D-Tx-mass} holds for $D=X_0\setminus S$.

For any $x\in D$ we have $\overline{B}_x\le \frac{T_x}{2}$, and since
$n_x = T_x-b_x$ pathwise,
\[
  \mathbb{E}[n_x] = T_x - \overline{B}_x \ge \frac{T_x}{2},
\]
establishing \eqref{eq:dense-Ex}.
\end{proof}

We can now apply Proposition~\ref{prop:binomial-deviation} to each dense context.

\begin{lemma}\label{lem:sum-Nx-lower} 
Assume $\mathbb{E}[B_T]\le \frac{T}{4}$. For some $c>0$ (independent of $T$ and $m$), the
following holds:
\[
  \mathbb{E}\Big[\sum_{x\in X_0} |N_x|\Big]
  \;\ge\; c \sqrt{m_0 T}.
\]
\end{lemma}

\begin{proof}
Let $D\subseteq X_0$ be the set of dense contexts from Lemma~\ref{lem:dense-contexts}.
For each $x\in D$, the process $\{Z^{(x)}_i\}$, the indicators $\{I^{(x)}_i\}$,
and the horizon $L=T_x$ satisfy the assumptions of
Proposition~\ref{prop:binomial-deviation} with $\sigma^2=3/16$ and
$\alpha = \frac{1}{2}$ (so that $\mathbb{E}[n_x]\ge \frac{T_x}{2}$).
Applying the proposition with $L=T_x$ and this value of $\alpha$ yields
\(
  \mathbb{E}|N_x| \ge c_{\sigma,\alpha}\sqrt{T_x},
\)
for a constant $c_{\sigma,\alpha}>0$ depending only on the variance lower
bound and $\alpha$. Hence there exists a constant $c'>0$  such that
\[
  \mathbb{E}|N_x|
  = \mathbb{E}\Big|\sum_{i=1}^{T_x} I^{(x)}_i Z^{(x)}_i\Big|
  \;\ge\; c' \sqrt{T_x}
  \qquad\text{for all }x\in D.
\]

We therefore have
\[
  \mathbb{E}\Big[\sum_{x\in X_0} |N_x|\Big]
  \;\ge\; \sum_{x\in D} \mathbb{E}|N_x|
  \;\ge\; c \sum_{x\in D} \sqrt{T_x}.
\]
Using \eqref{eq:Tx-bounds} and \eqref{eq:D-Tx-mass}, we obtain
\[
  \sum_{x\in D} \sqrt{T_x}
  \;\ge\; \sqrt{T_x^{\min}} \sum_{x\in D} 1
  \;\ge\; \sqrt{T_x^{\min}}\,
          \frac{\sum_{x\in D} T_x}{T_x^{\max}}
  \;\ge\; \sqrt{c_{\mathrm{occ}}\frac{T}{m_0}}\,
          \frac{T/2}{C_{\mathrm{occ}} T/m_0}
  = \frac{\sqrt{c_{\mathrm{occ}}}}{2C_{\mathrm{occ}}}\,\sqrt{m_0 T}.
\]
(Here $T_x^{\min}$ and $T_x^{\max}$ denote the minimum and maximum of the
$T_x$, controlled by \eqref{eq:Tx-bounds}.)
Thus
\[
  \mathbb{E}\Big[\sum_{x\in X_0} |N_x|\Big]
  \;\ge\; c'\,\frac{\sqrt{c_{\mathrm{occ}}}}{2C_{\mathrm{occ}}}\,\sqrt{m_0 T}
  =: c \sqrt{m_0 T},
\]
as claimed.
\end{proof}

\subsubsection{Lower Bound for the ``Honest'' Group}

We can now combine the previous bounds to obtain a lower bound on the $g_3$
contribution.

\begin{lemma}\label{lem:g4}
Assume $T\ge m_0$ and $\eta\le 1/(2m)$. Then there exists a constant $c>0$ (independent of $T$ and $m$) such that, for any deterministic prediction algorithm under $\mathcal{D}_{T,m}$,
either
\begin{equation}\label{eq:g4-case1}
  \mathbb{E}[B_T] \;\ge\; \frac{T}{4}
  \qquad\text{and hence}\qquad
  \mathbb{E}[\mathrm{MCerr}_T]
  \;\ge\; \frac{\eta}{8}\,T,
\end{equation}
or else $\mathbb{E}[B_T]< \frac{T}{4}$ and
\begin{equation}\label{eq:g4-case2}
  \mathbb{E}\Big[\sum_{v\in V_T} |B_T(v,g_3)|\Big]
  \;\ge\; c \sqrt{m_0 T} - \eta T.
\end{equation}
\end{lemma}

\begin{proof}
If $\mathbb{E}[B_T]\ge \frac{T}{4}$, then Lemma~\ref{lem:big-lies} yields
\[
  \mathbb{E}[\mathrm{MCerr}_T]
  \;\ge\; \frac{\eta}{2}\,\mathbb{E}[B_T]
  \;\ge\; \frac{\eta}{8}\,T,
\]
which is \eqref{eq:g4-case1}.

Otherwise, if $\mathbb{E}[B_T]< \frac{T}{4}$, then Lemma~\ref{lem:sum-Nx-lower} gives
\[
  \mathbb{E}\Big[\sum_{x\in X_0} |N_x|\Big]
  \;\ge\; c \sqrt{m_0 T},
\]
and plugging this into \eqref{eq:g4-main-reduction} yields
\[
  \mathbb{E}\Big[\sum_{v\in V_T} |B_T(v,g_3)|\Big]
  \;\ge\; c \sqrt{m_0 T} - \eta T.
\]
This gives us \eqref{eq:g4-case2}. 
\end{proof} 

\subsubsection{Putting it All Together}

We now combine all the pieces to prove Theorem~\ref{thm:main}. The argument proceeds by case analysis:
\begin{itemize}
    \item If the algorithm makes many big deviations ($\mathbb{E}[B_T] \ge T/4$), then Lemma~\ref{lem:big-lies} gives $\mathbb{E}[\mathrm{MCerr}_T] \ge \frac{\eta}{8} T = \Omega(T^{2/3})$.
    \item If the algorithm makes few big deviations ($\mathbb{E}[B_T] < T/4$), then most rounds are $\eta$-honest. Lemma~\ref{lem:g4} shows the noise contribution on group $g_3$ is $\Omega(\sqrt{mT}) = \Omega(T^{2/3})$.
\end{itemize}
The threshold $\eta = \Theta(\sqrt{m/T}) = \Theta(T^{-1/3})$ is chosen to balance these two cases.

  Fix $T\ge 1$ and let $m:=\lfloor T^{1/3}\rfloor$.
  For $T$ sufficiently large, we have $m\ge 8$ and $m\le T^{1/3}$. 

Let $c_0>0$ be the density constant from the hard distribution construction (so $m_0\ge c_0 m$), and let $c>0$ be the constant from
Lemma~\ref{lem:g4}.
Set
\[
  \eta := \delta \sqrt{\frac{m}{T}},
\]
where $\delta$ is a constant $>0$ such that: 
\begin{equation}
  \delta \;\le\; \frac{c\sqrt{c_0}}{4}
  \qquad\text{and}\qquad
  \delta \;\le\; \frac{1}{2}.
  \label{eq:delta-constraints}
\end{equation}
Using $m_0\ge c_0 m$, we have for all sufficiently large $T$,
\[
  \eta T = \delta\sqrt{mT}
  \;\le\; \delta\,\frac{1}{\sqrt{c_0}}\sqrt{m_0 T}
  \;\le\; \frac{c}{4}\sqrt{m_0 T}
  \;\le\; \frac{c}{2}\sqrt{m_0 T},
\]
so the drift term $\eta T$ is always at most $\frac{c}{2}\sqrt{m_0 T}$.
Moreover, since $m\le T^{1/3}$ we have $\eta = \delta\sqrt{m/T}\le \delta T^{-1/3}$, and the bound
$\delta\le 1/2$ in~\eqref{eq:delta-constraints} yields
\[
  \eta \;\le\; \frac{1}{2}T^{-1/3} \;\le\; \frac{1}{2m}
\]
for all sufficiently large $T$, so the condition $\eta\le 1/(2m)$ needed in Lemmas~\ref{lem:g4-context-decomp} and~\ref{lem:g4} also holds.

\medskip

Let $A$ be any prediction algorithm.
By Lemma~\ref{lem:deterministic}, we may assume without loss of generality that $A$ is deterministic. We distinguish two cases according to the size of $\mathbb{E}[B_T]$.

\paragraph{Case 1: $\mathbb{E}[B_T]\ge \frac{T}{4}$.}
In this case, Lemma~\ref{lem:big-lies} yields
\[
  \mathbb{E}[\mathrm{MCerr}_T]
  \;\ge\; \frac{\eta}{2}\,\mathbb{E}[B_T]
  \;\ge\; \frac{\eta}{8}\,T
  = \frac{\delta}{8}\sqrt{mT}.
\]
Since $m =  \Theta(T^{1/3})$, this gives
\[
  \mathbb{E}[\mathrm{MCerr}_T] =  \Omega(T^{2/3}).
\]

\paragraph{Case 2: $\mathbb{E}[B_T]< \frac{T}{4}$.}
In this case, Lemma~\ref{lem:g4} (specifically \eqref{eq:g4-case2}) gives
\[
  \mathbb{E}\Big[\sum_{v\in V_T} |B_T(v,g_3)|\Big]
  \;\ge\; c \sqrt{m_0 T} - \eta T
  \;\ge\; \frac{c}{2}\sqrt{m_0 T},
\]
where the last inequality uses $\eta T \le (c/2)\sqrt{m_0 T}$ from our choice of $\delta$ in
\eqref{eq:delta-constraints}. Thus the drift term subtracts at most half of the noise lower bound, and the
right-hand side remains bounded below by a constant multiple of $\sqrt{m_0 T}$.
Since $m_0\ge c_0 m$ for a constant $c_0>0$, we have
\[
  \sqrt{m_0 T} \;\ge\; \sqrt{c_0}\,\sqrt{mT} = \Theta(\sqrt{mT}),
\]
and therefore
\[
  \mathbb{E}\Big[\sum_{v\in V_T} |B_T(v,g_3)|\Big]
  \;\ge\; \frac{c\sqrt{c_0}}{2}\sqrt{mT}
  = \Omega(\sqrt{mT}) = \Omega(T^{2/3}).
\]
Since $\mathrm{MCerr}_T$ is at least the $g_3$ contribution, this again implies
\[
  \mathbb{E}[\mathrm{MCerr}_T] = \Omega(T^{2/3}).
\]

Combining the two cases, we have established
\[
  \mathbb{E}[\mathrm{MCerr}_T] \;\ge\; c'\,T^{2/3}
\]
for some constant $c'>0$ and all sufficiently large $T$ (depending only on the fixed constants
$c_0,c,\delta$). 

Thus, even for a family of only three simple groups, no online prediction algorithm can guarantee
expected multicalibration error $o(T^{2/3})$ against adversarially chosen contexts and outcomes,
matching (up to logarithmic factors) known online multicalibration upper bounds~\citep{noarov2023high}, and separating the statistical complexity of multicalibration from marginal calibration for which $O(T^{2/3-\varepsilon})$ upper bounds are known \citep{dagan2025breaking}.

\section{Optimal Lower Bound for Prediction-Independent Groups}
\label{sec:walsh-localrate-nonneg}

This section gives a lower bound for online multicalibration using only
\emph{prediction-independent} group functions $g:\mathcal{C}\to[0,1]$, where $\mathcal{C}$ is the context space.

Specifically, this section is devoted to proving the following theorem: 
\begin{theorem}[Prediction-independent lower bound]
\label{thm:walsh-localrate-nonneg}
There exist universal constants $c,C>0$ and $T_0\in\mathbb{N}$ such that for all $T\ge T_0$,
under the distribution and prediction-independent group family $G$ defined in this section,
every (possibly randomized) online forecaster satisfies
\[
\mathbb{E}\big[\mathrm{MCerr}_{T}(G)\big]
\;\ge\;
c\cdot \frac{T^{2/3}}{\log^C(T+1)}.
\]
\end{theorem}

\paragraph{Proof roadmap.} The key challenge in proving lower bounds for prediction-independent groups is
that we cannot directly detect when the learner deviates from ``honest'' predictions
\(p_t=\mathbb{E}[y_t]=x_t\), since our groups cannot depend on the prediction value. To be able to nevertheless control the learner's deviation from truthfulness, we proceed by defining and leveraging a prediction-independent family of \emph{Walsh groups}. Now, we go over the main proof steps.

After defining
the hard distribution and the Walsh group family, we begin by extracting the lower bound
from the constant group \(g_{\mathrm{all}}\) via a bias--noise decomposition (Section~\ref{sec:bias-noise}). Namely,
for
\[
A := \sum_{t\le T}|p_t-x_t|,
\qquad
U_v := \sum_{t\le T:p_t=v}(x_t-y_t),
\qquad
n_v := |\{t\le T:p_t=v\}|,
\qquad
N := \sum_{v\in V_T}\sqrt{n_v},
\]
we show that
\[
\mathrm{Err}_T(g_{\mathrm{all}})
\ge
\sum_{v\in V_T}|U_v| - A.
\]
The rest of the proof bounds the two terms in this display, and proceeds via the following steps.

1. \textbf{Walsh groups enforce approximate honesty and control the bias} (Section~\ref{sec:walsh-l1-truthfulness}): We first use
the orthogonality and prefix-sum properties of the auxiliary full Walsh family to show that, for
any prediction value \(v\), the discrete sign pattern \(\mathrm{sign}(v-x_i)\) on the mean grid can
be exactly represented as a linear combination of the full Walsh family, with tightly bounded
total coefficient weight. We then show that there exists a polylogarithmic-size subsampled Walsh
family that approximately spans these discrete sign patterns with the same bound on total
coefficient weight. Based upon this property, any forecaster with small multicalibration error on
the subsampled Walsh family must have small total \(\ell_1\)-deviation from honest predictions:
\[
\sum_{t\le T}|p_t-x_t|=\widetilde{O}(\mathrm{MCerr}_T).
\]

2. \textbf{Approximate honesty forces diverse predictions} (Section~\ref{sec:diverse}): 
Small \(\ell_1\)-deviation from
honesty prevents the forecaster from concentrating predictions on a few values, forcing the
diversity parameter \(N=\sum_v\sqrt{n_v}\) to be large, where \(n_v\) counts predictions equal
to \(v\).

3. \textbf{Adaptive bucketing lower-bounds the noise} (Section~\ref{sec:noise-control}): As we then demonstrate, the noise term
$\sum_{v\in V_T}\bigg|\sum_{t:p_t=v}(x_t-y_t)\bigg|$
is lower-bounded, up to logarithmic factors, by the diversity parameter \(N\) --- and this holds true even if the
forecaster tries to prevent that by choosing the buckets adaptively based on past noise realizations. Proving that such adaptive bucketing is harmless turns out to be quite technically involved.

Finally, in Section~\ref{sec:final-proof} we combine the bias-noise decomposition for the constant group with the bias/truthfulness bound,
the diversity bound, and the adaptive-bucketing noise bound, and in this way obtain the desired
\(\widetilde{\Omega}(T^{2/3})\) multicalibration error lower bound.

\vspace{0.3in}

For what follows, recall that for a group function $g:\mathcal{C}\to[0,1]$, its calibration error is
\[
\mathrm{Err}_T(g) := \sum_{v\in \mathcal{V}_T}\left|\sum_{t=1}^T \mathbf{1}[p^t=v]\;g(x^t)\,(p^t-y^t)\right|,
\]
where $\mathcal{V}_T:=\{p^1,\dots,p^T\}$. 
In the analysis we will also apply this definition to signed weight functions $w:\mathcal{C}\to[-1,1]$ using the same formula.

\newcommand{\idx}{\mathrm{idx}}

\subsection{The Hard Distribution}

Our hard distribution is similar to the one we used in Section \ref{sec:prediction-dependent}: the contexts are fixed deterministically, and encode label means uniformly spread throughout $[1/4, 3/4]$. The change is that the labels are no longer Bernoulli, but result from adding Rademacher noise to the label mean (which simplifies some of our arguments).

Fix a horizon $T\ge 2$. Let $m:=\max\{2,2^{\lfloor \log_2(T^{1/3}) \rfloor}\}$ and define grid means
\[
x_i := \frac14 + \frac{i-1}{2(m-1)}\qquad (i=1,\dots,m),
\]
so $x_i\in[1/4,3/4]$ and $|x_{i+1}-x_i|=\Theta(1/m)$.
Define the deterministic round-robin mean sequence
\[
x^t := x_{\,1+((t-1)\bmod m)}\qquad (t=1,\dots,T).
\]
Outcomes are generated by adding independent Rademacher noise to the means.
Let $(\xi^t)_{t=1}^T$ be independent signs with $\mathbb{P}(\xi^t=1)=\mathbb{P}(\xi^t=-1)=1/2$, and set
\begin{equation}
\label{eq:env-walsh}
  y^t := x^t + \frac{\xi^t}{4},\qquad t=1,\dots,T.
\end{equation}
Then $y^t\in[0,1]$ for all $t$, $\mathbb{E}[y^t\mid x^t]=x^t$, and $\mathrm{Var}(y^t\mid x^t)=1/16$.

\subsection{The Hard Group Family}

\subsubsection{Walsh System}
We will construct our prediction-independent groups from the $\{\pm 1\}$-valued Walsh system on the mean grid. Orthogonality of the Walsh system lets us test the forecaster's behavior in many independent directions, while a prefix-sum property (Lemma~\ref{lem:walsh-prefix}) can let us control $\ell_1$ deviation from honest predictions. Later we will show that a good polylogarithmic-sized subsample of the full Walsh system suffices for this purpose in Lemma \ref{lem:walsh-subsample}.
\begin{definition}[Walsh system]
    Let $n$ be a power of two. 
    For any $j \in \{0,\ldots,n-1\}$, $s \in \{0,\ldots,n-1\}$, write $j$ and $s$ in binary and let $\langle j,s\rangle_2$ denote their mod-$2$ inner product. Define
    \[
        \psi^\Walsh_j : \{0,\dots,n-1\} \to \{\pm 1\}, \qquad \psi^\Walsh_j(s) := (-1)^{\langle j,s\rangle_2}.
    \]
    The resulting collection of functions $\{\psi^\Walsh_j\}_{j=0}^{n-1}$ is called a \emph{Walsh system of length $n$}.
\end{definition}

The Walsh system is $\{\pm1\}$-valued and orthogonal: for all $j,j'\in\{0,\dots,n-1\}$,
\[
\sum_{s=0}^{n-1} \psi^\Walsh_{j}(s)\,\psi^\Walsh_{j'}(s) = n\cdot \mathbf{1}[j=j'].
\]
In addition, it satisfies a prefix-sum bound as stated in Lemma \ref{lem:walsh-prefix}. This additional property will be crucial for showing that calibration error on the Walsh groups controls the $\ell_1$ distance to honest predictions.

\begin{lemma}[Walsh prefix-sum bound]\label{lem:walsh-prefix}
    For every $j \in \{1,\dots,n-1\}$, denote the number of trailing zeroes in its binary expansion as
    \[
        \tz(j)\;:=\;\max\{d \ge 0: \ 2^d \text{ divides } j\}.
    \]
    Then we have
    \[
        \max_{r\in\{0,\dots,n\}}\left|\sum_{s<r} \psi^\Walsh_j(s)\right|\ \le\ 2^{\tz(j)}.
    \]
\end{lemma}

\begin{proof}
Fix $j\in\{1,\dots,n-1\}$.
Write the binary expansions
\[
j=\sum_{b \ge 0} j_b 2^b,\qquad s=\sum_{b \ge 0} s_b 2^b,
\]
where $j_b,s_b\in\{0,1\}$ denote the $b$-th (least-significant) bits.

By definition of $\tz$, we have
\[
    j_0=j_1=\cdots=j_{\tz(j)-1}=0
    \qquad\text{and}\qquad
    j_{\tz(j)}=1.
\]

As a result,
\[
    \langle j,s \rangle_2 = \sum_{b \ge 0} j_b s_b = \sum_{b \ge \tz(j)} j_b s_b \pmod 2.
\]

Therefore, $\psi^\Walsh_j(s)$ doesn't depend on the first $\tz(j)$ bits of $s$. That is to say, $\psi^\Walsh_j$ is constant on each dyadic block of length $2^{\tz(j)}$:
\[
\{k\cdot 2^{\tz(j)},\ k\cdot 2^{\tz(j)}+1,\ \dots,\ (k+1)\cdot 2^{\tz(j)}-1\}.
\]

Next we group two consecutive blocks into a superblock of length $2^{\tz(j)+1}$.
Consider the two blocks:
\[
I_{k,0}:=\{2k\cdot 2^{\tz(j)},\dots,(2k+1)\cdot 2^{\tz(j)}-1\},\qquad
I_{k,1}:=\{(2k+1)\cdot 2^{\tz(j)},\dots,(2k+2)\cdot 2^{\tz(j)}-1\}.
\]

For any $s\in I_{k,0}$, the $\tz(j)$-th bit satisfies $s_{\tz(j)}=0$, while for any $s'\in I_{k,1}$ we have $s'_{\tz(j)}=1$.
In addition, $s_b=s'_b$ for all $b>\tz(j)$, i.e., the higher bits are the same across the two halves of a superblock.

Since $j_{\tz(j)}=1$, this implies
\[
    \langle j,s' \rangle_2 - \langle j,s \rangle_2 = \sum_{b \ge \tz(j)} j_b s'_b - \sum_{b \ge \tz(j)} j_b s_b = j_{\tz(j)}(s'_{\tz(j)} - s_{\tz(j)}) = 1 \pmod 2,
\]
and hence
\[
    \psi^\Walsh_j(s')=-\psi^\Walsh_j(s).
\]

Therefore the sum of $\psi^\Walsh_j(s)$ over a full superblock cancels:
\[
\sum_{s\in I_{k,0} \cup I_{k,1}}\psi^\Walsh_j(s)=0.
\]

Every prefix sum can be decomposed into a disjoint union of complete superblocks (each contributing $0$)
plus a remainder that contributes at most $2^{\tz(j)}$. Hence
\[
\left|\sum_{s<r}\psi^\Walsh_j(s)\right|\le 2^{\tz(j)}
\qquad\text{for all }r\in\{0,\dots,n\}.
\]
\end{proof}

\subsubsection{The Full and Subsampled Walsh Families}
We now define the prediction-independent group family used in our lower bound. The family consists of two types of groups:
\begin{itemize}
    \item A \textbf{constant group} $g_{\mathrm{all}}$ that enforces marginal calibration.
    \item \textbf{Walsh groups} defined on the mean grid. We will first build a full Walsh family, and later choose a good subsample as the final Walsh groups. Together with the constant group, this final family enforces that the forecaster's predictions stay close to the honest predictions in an $\ell_1$ sense.
\end{itemize}

\begin{definition}[Grid index map]
Define $\idx:[0,1]\to\{1,\dots,m\}$ by $\idx(x_i)=i$ for each $i=1,\dots,m$, and define $\idx(x)=1$
for all other $x\in[0,1]$.
\end{definition}

\paragraph{Constant group.}
We include the constant (marginal) group
\[
g_{\mathrm{all}}(x):=1.
\]

\paragraph{Walsh groups.}
For each $\ell\in\{1,\dots,m-1\}$ we define the signed Walsh feature
\[
    w_\ell(x)\ :=\ \psi^\Walsh_\ell\big(\idx(x)-1\big)\ = (-1)^{\langle \ell, \idx(x)-1 \rangle_2} \in\ \{\pm1\},
\]
and convert it to two binary Walsh half-groups by
\begin{align}
    g^{\Walsh,+}_\ell(x) &:= \frac{1+w_\ell(x)}{2} \in \{0,1\}, \\
    g^{\Walsh,-}_\ell(x) &:= \frac{1-w_\ell(x)}{2} \in \{0,1\}.
\end{align}

(We omit $\ell=0$ since $w_0\equiv 1$ is already covered by $g_{\mathrm{all}}$; we will sometimes identify $g_{\mathrm{all}}$ with $w_0$ for convenience.)

Figure~\ref{fig:walsh-probes-online} shows two sample Walsh sign patterns on the ordered mean grid together with the corresponding positive and negative half-groups.

\paragraph{The full Walsh family.}
Let
\[
G^{\mathrm{full}}
:=
\{g_{\mathrm{all}}\}
\ \cup\ 
\{g^{\Walsh,+}_\ell,\ g^{\Walsh,-}_\ell:\ \ell=1,\dots,m-1\}
\]

This family consists of $1 + 2(m-1) = 2m-1 = \Theta(T^{1/3})$ binary prediction-independent groups.
This is an auxiliary family and the final hard group family will be a much smaller subsampled Walsh family.

\begin{figure}[t]
\centering
\begin{tikzpicture}[x=0.85cm,y=0.95cm]
  \definecolor{walshplus}{RGB}{63,109,214}
  \definecolor{walshminus}{RGB}{225,138,55}

  \foreach \i in {1,...,16} {
    \node[font=\scriptsize] at (\i,2.95) {$\i$};
  }
  \node[anchor=east, font=\small] at (0.25,2.95) {$i$};

  \node[anchor=east, font=\small] at (0.25,2.1) {$w_1(x_i)$};
  \node[anchor=east, font=\small] at (0.25,0.95) {$w_4(x_i)$};
  \foreach \i/\sign/\fillcol in {
    1/{+}/walshplus,2/{-}/walshminus,3/{+}/walshplus,4/{-}/walshminus,
    5/{+}/walshplus,6/{-}/walshminus,7/{+}/walshplus,8/{-}/walshminus,
    9/{+}/walshplus,10/{-}/walshminus,11/{+}/walshplus,12/{-}/walshminus,
    13/{+}/walshplus,14/{-}/walshminus,15/{+}/walshplus,16/{-}/walshminus} {
    \draw[draw=black!35, fill=\fillcol!28, rounded corners=1pt] (\i-0.44,1.65) rectangle (\i+0.44,2.55);
    \node[font=\small] at (\i,2.1) {$\sign$};
  }

  \foreach \i/\sign/\fillcol in {
    1/{+}/walshplus,2/{+}/walshplus,3/{+}/walshplus,4/{+}/walshplus,
    5/{-}/walshminus,6/{-}/walshminus,7/{-}/walshminus,8/{-}/walshminus,
    9/{+}/walshplus,10/{+}/walshplus,11/{+}/walshplus,12/{+}/walshplus,
    13/{-}/walshminus,14/{-}/walshminus,15/{-}/walshminus,16/{-}/walshminus} {
    \draw[draw=black!35, fill=\fillcol!28, rounded corners=1pt] (\i-0.44,0.5) rectangle (\i+0.44,1.4);
    \node[font=\small] at (\i,0.95) {$\sign$};
  }

  \draw[draw=black!35, fill=walshplus!28, rounded corners=1pt] (2.8,-0.35) rectangle (3.35,0.15);
  \node[anchor=west, font=\small] at (3.55,-0.1) {$g^{\Walsh,+}_\ell(x_i)=1$};
  \draw[draw=black!35, fill=walshminus!28, rounded corners=1pt] (9.0,-0.35) rectangle (9.55,0.15);
  \node[anchor=west, font=\small] at (9.75,-0.1) {$g^{\Walsh,-}_\ell(x_i)=1$};
\end{tikzpicture}
\caption{Two example Walsh sign patterns on the ordered mean grid. Blue cells belong to the positive Walsh half-group and orange cells belong to the negative Walsh half-group, so $w_\ell(x_i)=g^{\Walsh,+}_\ell(x_i)-g^{\Walsh,-}_\ell(x_i)$. Different values of $\ell$ supply dyadic sign patterns at different scales; the full Walsh family contains such probes at all scales, and later we will keep only a carefully chosen subsample.}
\label{fig:walsh-probes-online}
\end{figure}

\paragraph{The subsampled Walsh family.}
For a set $S\subseteq\{1,\dots,m-1\}$, let
\[
G(S)
:=
\{g_{\mathrm{all}}\}
\ \cup\ 
\{g^{\Walsh,+}_\ell,\ g^{\Walsh,-}_\ell:\ \ell\in S\}.
\]
Later we will choose a set $S$ with $|S|=O(\log^3 m)=O(\log^3 T)$ and define the final hard group family by $G:=G(S)$.

\paragraph{Simulating signed weights by differences.}
Our groups take values in $\{0,1\}$, but it will be useful to consider calibration errors with respect to weighting functions
that can take negative values.
To be specific, we will consider calibration error with respect to the signed Walsh feature, which takes the difference between two groups in $G(S)$ whenever $\ell\in S$:
\begin{equation}
\label{eq:signed-aux-walsh}
    w_\ell = g^{\Walsh,+}_\ell - g^{\Walsh,-}_\ell \in \{\pm 1\}.
\end{equation}

\begin{lemma}[Difference-of-two reduction]
\label{lem:diff-two}
Fix any two groups $g^+,g^-:\mathcal{C}\to[0,1]$. For every realization,
\[
\mathrm{Err}_{T}(g^+-g^-) \;\le\; \mathrm{Err}_{T}(g^+) + \mathrm{Err}_{T}(g^-),
\]
and consequently
\[
\max\{\mathrm{Err}_{T}(g^+),\mathrm{Err}_{T}(g^-)\}\ \ge\ \tfrac12\,\mathrm{Err}_{T}(g^+-g^-).
\]
In particular, for every $\ell\in S$, with $w_{\ell}$ as in \eqref{eq:signed-aux-walsh},
\[
\mathrm{MCerr}_{T}(G(S)) \;\ge\; \tfrac12\,\mathrm{Err}_{T}(w_{\ell}).
\]
\end{lemma}
\begin{proof}
Fix a prediction value $v$. Let
\[
B(v,g):=\sum_{t\le T} \mathbf{1}[p^t=v]\;g(x^t)\,(p^t-y^t).
\]
Then $B(v,g^+-g^-)=B(v,g^+)-B(v,g^-)$, so $|B(v,g^+-g^-)|\le |B(v,g^+)|+|B(v,g^-)|$.
Summing over $v$ gives $\mathrm{Err}_{T}(g^+-g^-)\le \mathrm{Err}_{T}(g^+)+\mathrm{Err}_{T}(g^-)$, and hence the max is at least half.
For the last inequality, apply this to the pair $(g^{\Walsh,+}_{\ell},g^{\Walsh,-}_{\ell})$ for the chosen $\ell\in S$.
\end{proof}

\subsection{The Constant-Group Bias-Noise Decomposition}
\label{sec:bias-noise}



The final lower bound of our theorem will be extracted from the constant group $g_\mathrm{all}$. As we will formalize in a moment, its calibration error $\operatorname{Err}_T(g_{\rm all})$ separates into a bias part, which comes from deviations from the honest predictor \(x_t\), and a noise part, which comes from the randomized labels. If we can prove that $\operatorname{Err}_T(g_{\rm all})$ is large, the same will be true for the multicalibration error.

Now, we will introduce some notation that will be used to accomplish this in the rest of the section. For the upcoming definition, recall that the prediction values are \(V_T=(p_t)_{t \in [T]}\).

\begin{definition}[Truthfulness, Noise, and Diversity]
Define the predictor's $\ell_1$-truthfulness, or the total $\ell_1$-deviation from the honest predictor, by
\[
A:=\sum_{t\le T}|p_t-x_t|.
\]
Next, define the label-noise increments
$Z_t:=x_t-y_t$
and, for each value \(v\in V_T\), the bucketed noise
\[
U_v:=\sum_{t\le T:p_t=v} Z_t
=
\sum_{t\le T:p_t=v}(x_t-y_t).
\]
Finally, define the prediction-bucket counts $n_v$ and the \emph{prediction diversity} $N$ as:
\[
n_v:=|\{t\le T:p_t=v\}|,
\qquad
N:=\sum_{v\in V_T}\sqrt{n_v}.
\]
\end{definition}

\begin{observation}[Constant-group bias--noise decomposition]
\label{obs:constant-group-bias-noise}
For every realization,
\[
\operatorname{Err}_T(g_{\rm all})
\ge
\sum_{v\in V_T}|U_v|-A.
\]
\end{observation}

\begin{proof}
Since \(g_{\rm all}\equiv 1\), using \(p_t-y_t=(p_t-x_t)+(x_t-y_t)\) and the reverse triangle inequality,
\[
\operatorname{Err}_T(g_{\rm all})
=
\sum_{v\in V_T}
\bigg|
\sum_{t\le T:p_t=v}(p_t-y_t)
\bigg|
=
\sum_{v\in V_T}
\bigg|
\sum_{t\le T:p_t=v}(p_t-x_t)
+
U_v
\bigg|
\ge
\sum_{v\in V_T}|U_v|
-
\sum_{v\in V_T}
\bigg|
\sum_{t\le T:p_t=v}(p_t-x_t)
\bigg|.
\]
Now, it suffices to observe that
\[
\sum_{v\in V_T}
\bigg|
\sum_{t\le T:p_t=v}(p_t-x_t)
\bigg|
\le
\sum_{v\in V_T}
\sum_{t\le T:p_t=v}|p_t-x_t|
=
A. \qedhere
\]
\end{proof}

The rest of the proof revolves around this decomposition. Section~\ref{sec:walsh-l1-truthfulness} bounds the bias penalty \(A\) by leveraging the multicalibration error with respect to the Walsh groups. Section~\ref{sec:diverse} shows that small \(A\) forces large prediction diversity, as quantified by \(N\). Section~\ref{sec:noise-control} then converts large prediction diversity into  a lower bound on the noise term \(\sum_v |U_v|\) via the Adaptive Bucketing theorem. Finally, Section~\ref{sec:final-proof} combines these ingredients.

\subsection{Bounding the Bias: Walsh Groups Enforce \texorpdfstring{$\ell_1$}{L1}-Truthfulness}
\label{sec:walsh-l1-truthfulness}

We now bound the $\ell_1$-truthfulness, or bias penalty, term \(A\) appearing in
Observation~\ref{obs:constant-group-bias-noise}. Specifically, we show that multicalibration
with respect to a carefully chosen subsampled Walsh family forces the forecaster to be close to
the honest predictor in total \(\ell_1\) loss.

The high-level idea is as follows. For any prediction value $v$, the sign pattern $\mathrm{sign}(v - x_i)$ across grid points $i$ can be expanded exactly by the full Walsh family. The Walsh prefix-sum bound (Lemma~\ref{lem:walsh-prefix}) ensures that this exact expansion has small $\ell_1$ coefficient mass---at most $O(\log m)$. We then show that there exists a subsampled Walsh family of polylogarithmic size that approximates all of these sign patterns simultaneously, while maintaining the small $\ell_1$ coefficient mass. Therefore, the total $\ell_1$ deviation $\sum_t |p^t - x^t|$ can be approximately written as a weighted combination of calibration biases on the subsampled Walsh groups, with total weight $O(\log m)$. If all these biases are small (i.e., if multicalibration error is small), then the $\ell_1$ deviation must also be small.


\subsubsection{Walsh Expansion of Discrete Threshold Signs on the Grid}

Fix a prediction value $v\in[0,1]$.
Because the context means always lie on the grid $\{x_1,\dots,x_m\}$, the sign pattern $\mathrm{sign}(v-x_i)$ is determined solely by the number of grid points $\le v$.
Define
\[
    r(v) := \bigl|\{i\in\{1,\dots,m\}: x_i \le v\}\bigr|\in\{0,1,\dots,m\}.
\]
Define the discrete sign function on indices $u\in\{0,\dots,m-1\}$ by
\[
    f_r(u) :=
    \begin{cases}
        +1 & \text{if } u\le r-1,\\
        -1 & \text{if } u\ge r,
    \end{cases}
    \qquad r\in\{0,1,\dots,m\}.
\]
Then for every grid point $x_i$ we have $f_{r(v)}(i-1)=+1$ if $x_i\le v$ and $f_{r(v)}(i-1)=-1$ if $x_i>v$.
In particular, for every time $t\le T$ (with $x^t\in\{x_1,\dots,x_m\}$),
\[
    |v-x^t| = f_{r(v)}(\idx(x^t)-1)\cdot (v-x^t).
\]

We now expand the threshold sign pattern $f_r$ in the Walsh basis on $\{0,\dots,m-1\}$ and bound the $\ell_1$ mass of its coefficients.
Recall that $m=\max\{2,2^{\lfloor \log_2(T^{1/3}) \rfloor}\}$ is a power of two.

\begin{lemma}[Exact expansion of discrete threshold signs by the full Walsh family]
\label{lem:walsh-expansion}
Fix $m$ a power of two.
Let $\{\psi^\Walsh_\ell\}_{\ell=0}^{m-1}$ be the length-$m$ Walsh system. 
Then for every $r\in\{0,\ldots,m\}$, there exist coefficients $\{\alpha_\ell(r)\}_{\ell=0}^{m-1}$ such that
\[
    f_r(u) \;=\; \sum_{\ell=0}^{m-1} \alpha_\ell(r)\,\psi^\Walsh_\ell(u) \quad \text{for every } u\in\{0,\ldots,m-1\}.
\]
Moreover, the coefficients satisfy that
\[
    \max_{r \in \{0,\ldots,m\}} |\alpha_0(r)| \le 1, \quad \sum_{\ell=1}^{m-1} \max_{r \in \{0,\ldots,m\}} |\alpha_\ell(r)| \;\le\; \log_2m.
\]
\end{lemma}

\begin{proof}
Define the Walsh coefficients by
\[
    \alpha_\ell(r) := \frac{1}{m}\sum_{u=0}^{m-1} f_r(u)\,\psi^\Walsh_\ell(u), \qquad \ell=0,\dots,m-1.
\]
Then the expansion $f_r(u)=\sum_{\ell}\alpha_\ell(r)\psi^\Walsh_\ell(u)$ follows from orthogonality of the Walsh system.

It remains to bound the $\ell_1$ mass. For $\ell=0$ we have $\psi^\Walsh_0\equiv 1$, so
\[
    |\alpha_0(r)|=\left|\frac{1}{m}\sum_{u=0}^{m-1} f_r(u)\right| = \left|\frac{2r-m}{m}\right|\le 1.
\]
Now fix $\ell\in\{1,\dots,m-1\}$. Using the identity $f_r(u)=2\mathbf{1}[u\le r-1]-1$ and
$\sum_{u=0}^{m-1}\psi^\Walsh_\ell(u)=0$ (which holds for $\ell\neq 0$), we obtain
\[
    \alpha_\ell(r)
    = \frac{1}{m}\sum_{u=0}^{m-1} f_r(u)\psi^\Walsh_\ell(u)
    = \frac{2}{m}\sum_{u=0}^{r-1}\psi^\Walsh_\ell(u).
\]
Therefore, by the Walsh prefix-sum bound, i.e., Lemma~\ref{lem:walsh-prefix}, we have
\[
    \max_{r \in \{0,\ldots,m\}} |\alpha_\ell(r)|
    \le \frac{2}{m}\cdot 2^{\tz(\ell)}.
\]

Group indices by $d:=\tz(\ell)\in\{0,1,\dots,\log_2 m-1\}$.
There are exactly $m/2^{d+1}$ values of $\ell\in\{1,\dots,m-1\}$ with $\tz(\ell)=d$.
Hence,
\[
    \sum_{\ell=1}^{m-1} \max_{r \in \{0,\ldots,m\}} |\alpha_\ell(r)|
    \le \sum_{\ell=1}^{m-1} \frac{2}{m}\cdot 2^{\tz(\ell)}
    \le \sum_{d=0}^{\log_2 m-1} \frac{m}{2^{d+1}}\cdot \frac{2}{m}\cdot 2^d
    = \sum_{d=0}^{\log_2 m-1} 1
    = \log_2 m.
\]
\end{proof}

\medskip
\noindent
The previous lemma is an exact statement for the full Walsh family. We now show that a much smaller subsampled Walsh family can approximate all of these threshold signs simultaneously. Later we will set the approximation error $\varepsilon=\tfrac12$.

\begin{lemma}[Approximation of discrete threshold signs by a subsampled Walsh family]
\label{lem:walsh-subsample}
Fix $m$ a power of two.
Let $\{\psi^\Walsh_\ell\}_{\ell=0}^{m-1}$ be the length-$m$ Walsh system, and let $\{\alpha_\ell(r)\}_{\ell=0}^{m-1}$ be the coefficients from Lemma~\ref{lem:walsh-expansion}.
Then there exists a universal constant $C_{\mathrm{sub}}>0$ such that for every $\varepsilon\in(0,1)$ there exist a polylogarithmic-sized subset
\[
S \subseteq \{1,\dots,m-1\}, \quad |S| \le C_{\mathrm{sub}}\frac{\log^3 m}{\varepsilon^2},
\]
and coefficients $\{c_\ell(r)\}_{\ell\in S,\ r\in\{0,\dots,m\}}$ such that the Walsh groups with indices in $S$ approximately span the discrete threshold signs with error $\varepsilon$. More specifically:
\begin{align*}
    & \widehat f_r(u) :=
    \alpha_0(r) + \sum_{\ell\in S} c_\ell(r)\,\psi^\Walsh_\ell(u), \\
    & |\widehat f_r(u)-f_r(u)| \le \varepsilon
    \qquad \text{for all } r\in\{0,\dots,m\},\ u\in\{0,\dots,m-1\}, \\
    & \max_{r \in \{0,\ldots,m\}} |\alpha_0(r)| \le 1, \quad  \sum_{\ell\in S}\max_{r\in\{0,\dots,m\}} |c_\ell(r)| \le \log_2 m.
\end{align*}
\end{lemma}

\begin{proof}
For $\ell\in\{1,\dots,m-1\}$, define
\[
b_\ell := \max_{r\in\{0,\dots,m\}} |\alpha_\ell(r)|
\qquad\text{and}\qquad
B := \sum_{\ell=1}^{m-1} b_\ell.
\]
By Lemma~\ref{lem:walsh-expansion}, we have $B\le \log_2 m$.
Also $B>0$, since some threshold sign $f_r$ is nonconstant.

Fix $\varepsilon\in(0,1)$ and an integer $K\ge 1$ to be chosen later.
Sample indices $J_1,\dots,J_K$ independently from $\{1,\dots,m-1\}$ according to the distribution
\[
\mathbb{P}(J_s=\ell)=\frac{b_\ell}{B}.
\]
For each $r\in\{0,\dots,m\}$ define the random approximation
\[
\widetilde f_r(u)
:=
\alpha_0(r)
+
\frac{B}{K}\sum_{s=1}^K
\frac{\alpha_{J_s}(r)}{b_{J_s}}\,
\psi^\Walsh_{J_s}(u).
\]

Fix any pair $(r,u)$. The summands are i.i.d., and for each $s$, taking expectation over the randomness of the indices $J_1,\dots,J_K$,
\[
\mathbb{E}\!\left[
B\,\frac{\alpha_{J_s}(r)}{b_{J_s}}\,\psi^\Walsh_{J_s}(u)
\right]
=
\sum_{\ell=1}^{m-1} \alpha_\ell(r)\psi^\Walsh_\ell(u).
\]
Therefore
\[
\mathbb{E}[\widetilde f_r(u)]
=
\alpha_0(r)+\sum_{\ell=1}^{m-1}\alpha_\ell(r)\psi^\Walsh_\ell(u)
=
f_r(u).
\]
Also, each summand is bounded in absolute value by $B$, because
\[
\left|
B\,\frac{\alpha_{J_s}(r)}{b_{J_s}}\,\psi^\Walsh_{J_s}(u)
\right|
\le B.
\]
Hence Hoeffding's inequality gives
\[
\mathbb{P}\!\left(
|\widetilde f_r(u)-f_r(u)|>\varepsilon
\right)
\le
2\exp\!\left(-\frac{K\varepsilon^2}{2B^2}\right).
\]

There are only $(m+1)m$ pairs $(r,u)$. Thus by a union bound,
\[
\mathbb{P}\!\left(
\max_{r\in\{0,\dots,m\}}\max_{u\in\{0,\dots,m-1\}}
|\widetilde f_r(u)-f_r(u)|>\varepsilon
\right)
\le
2m(m+1)\exp\!\left(-\frac{K\varepsilon^2}{2B^2}\right).
\]
Therefore, for a sufficiently large universal constant $C_{\mathrm{sub}}$, if
\[
K \ge C_{\mathrm{sub}}\frac{B^2\log(m+1)}{\varepsilon^2},
\]
then the right-hand side $2m(m+1)\exp\!\left(-\frac{K\varepsilon^2}{2B^2}\right)$ is strictly smaller than $1$. Hence there exists a realization of
$J_1,\dots,J_K$ for which
\[
\max_{r\in\{0,\dots,m\}}\max_{u\in\{0,\dots,m-1\}}
|\widetilde f_r(u)-f_r(u)|\le \varepsilon.
\]
Fix any such realization, and let
\[
S := \{J_1,\dots,J_K\}
\]
be the set of distinct sampled elements of $\{1,\dots,m-1\}$.

For each $\ell\in S$ and $r\in\{0,\dots,m\}$ define
\[
c_\ell(r)
:=
\frac{B}{K}\frac{\alpha_\ell(r)}{b_\ell}\cdot
\#\{s\in\{1,\dots,K\}:J_s=\ell\}.
\]
Then, after merging duplicates,
\[
\widetilde f_r(u)
=
\alpha_0(r)+\sum_{\ell\in S} c_\ell(r)\,\psi^\Walsh_\ell(u),
\]
so we may set $\widehat f_r:=\widetilde f_r$.
Finally,
\[
\sum_{\ell\in S}\max_{r\in\{0,\dots,m\}} |c_\ell(r)|
\le
\frac{B}{K}
\sum_{\ell\in S}
\#\{s\in\{1,\dots,K\}:J_s=\ell\}
=
B
\le
\log_2 m.
\]
Since $|S|\le K$ and $B\le \log_2 m\le \log_2(m+1)$, enlarging $C_{\mathrm{sub}}$ if necessary gives
\[
|S|
\le
C_{\mathrm{sub}}\frac{\log^3 m}{\varepsilon^2}.
\qedhere
\]
\end{proof}

\subsubsection{Walsh Groups Enforce \texorpdfstring{$\ell_1$}{L1}-Truthfulness}

We now set $\varepsilon=\tfrac12$ in Lemma~\ref{lem:walsh-subsample}. Fix a set
\[
S \subseteq \{1,\dots,m-1\}
\]
and coefficients $\{c_\ell(r)\}_{\ell\in S,\ r\in\{0,\dots,m\}}$ satisfying the conclusion of that lemma for $\varepsilon=\tfrac12$, and define the final hard group family by
\[
G := G(S).
\]
This is a subsampled Walsh family of size
\[
|G| = 1+2|S| = O(\log^3 m) = O(\log^3 T).
\]
Now we use the Walsh groups to enforce a form of $\ell_1$ truthfulness.

\begin{lemma}[$\ell_1$-Truthfulness from Walsh groups]
\label{lem:l1-truthfulness}
For
$A := \sum_{t=1}^{T} |p^t-x^t|,$
the total $\ell_1$-deviation from honesty,
there is a universal $C_{\ell_1}>0$ such that, under the environment \eqref{eq:env-walsh},
every forecaster satisfies
\[
\mathbb{E}[A] \le C_{\ell_1}\log(m+1)\cdot \mathbb{E}[\mathrm{MCerr}_{T}(G)].
\]
\end{lemma}

\begin{proof}
Fix a forecaster. For each realized prediction value $v\in \mathcal{V}_{T}$, define the prediction-dependent approximate sign weight on contexts
\[
\widehat s_v(x)
:=
\alpha_0(r(v))
+
\sum_{\ell\in S} c_\ell(r(v))\,w_\ell(x).
\]
For any grid point $x_i$, Lemma~\ref{lem:walsh-subsample} gives
\[
\left|
\widehat s_v(x_i)-f_{r(v)}(i-1)
\right|
\le \frac12.
\]
Therefore
\begin{align*}
\widehat s_v(x_i)\,(v-x_i)
&=
f_{r(v)}(i-1)\,(v-x_i)
+
\bigl(\widehat s_v(x_i)-f_{r(v)}(i-1)\bigr)(v-x_i) \\
&\ge
|v-x_i|-\frac12|v-x_i|
=
\frac12|v-x_i|.
\end{align*}
In particular, for all $t\le T$ (with $x^t\in\{x_1,\dots,x_m\}$),
\[
\widehat s_{p^t}(x^t)\,(p^t-x^t)\ge \frac12|p^t-x^t|.
\]

Let $S_v := \{t\le T: p^t=v\}$. Then
\[
A 
= 
\sum_{t \le T} |p^t - x^t|
\le
2 \sum_{t\le T} \widehat s_{p^t}(x^t)\,(p^t-x^t)
=
2 \sum_{v\in \mathcal{V}_{T}}\sum_{t\in S_v} \widehat s_v(x^t)\,(v-x^t).
\]

Let $\mathcal{F}_t$ be the sigma-field generated by the transcript up to and including the realized prediction $p^t$, but excluding $y^t$.
Since the environment is oblivious, $y^t$ is independent of $\mathcal{F}_t$ and satisfies $\mathbb{E}[y^t\mid \mathcal{F}_t]=x^t$.
Thus $\mathbb{E}[p^t-y^t\mid \mathcal{F}_t]=p^t-x^t$.
Moreover $\widehat s_{p^t}(x^t)$ is $\mathcal{F}_t$-measurable, so
\begin{align*}
\mathbb{E}[A]
&\le
2\,
\mathbb{E}\left[\sum_{t\le T} \widehat s_{p^t}(x^t)(p^t - x^t)\right] \\
&=
2\,
\mathbb{E}\left[\sum_{t\le T} \widehat s_{p^t}(x^t) \, \mathbb{E}[p^t - y^t \mid \mathcal{F}_t]\right] \\
&=
2\,
\mathbb{E}\left[\sum_{t\le T} \widehat s_{p^t}(x^t)(p^t - y^t)\right] \\
&=
2\,
\mathbb{E}\left[\sum_{v\in \mathcal{V}_{T}} \sum_{t\in S_v} \widehat s_v(x^t)(v - y^t)\right].
\end{align*}

By definition, $\widehat s_v(x^t)$ can be expanded as
\[
\widehat s_v(x^t)
:=
\alpha_0(r(v))
+
\sum_{\ell\in S} c_\ell(r(v))\,w_\ell(x^t).
\]

Substituting into the sum gives
\begin{align*}
    \mathbb{E}[A]
    &\le 
    2\,
    \mathbb{E}\left[\sum_{v\in \mathcal{V}_{T}} \sum_{t\in S_v} \left( \alpha_0(r(v)) + \sum_{\ell\in S} c_\ell(r(v))\,w_\ell(x^t) \right) (v - y^t)\right] \\
    &=
    2\,
    \mathbb{E}\left[
    \sum_{v\in\mathcal{V}_T} \alpha_0(r(v))\,B_T(v,g_{\mathrm{all}})
    +
    \sum_{\ell\in S}\sum_{v\in\mathcal{V}_T} c_\ell(r(v))\,B_T(v,w_\ell)
    \right].
\end{align*}

By the triangle inequality,
\begin{align*}
    \mathbb{E}[A]
    &\le
    2\,
    \mathbb{E}\left[
    \max_{r\in\{0,\dots,m\}}|\alpha_0(r)|\;\mathrm{Err}_T(g_{\mathrm{all}})
    +
    \sum_{\ell\in S}\max_{r\in\{0,\dots,m\}}|c_\ell(r)|\;\mathrm{Err}_T(w_\ell)
    \right].
\end{align*}

Finally, by the fact that $g_{\mathrm{all}}\in G$ and the difference-of-two reduction (Lemma~\ref{lem:diff-two}), we have
\begin{align*}
    \mathrm{Err}_{T}(g_{\mathrm{all}}) &\le \mathrm{MCerr}_{T}(G), \\
    \mathrm{Err}_{T}(w_\ell) &\le 2\,\mathrm{MCerr}_{T}(G) \quad \text{for each } \ell \in S.
\end{align*}

Moreover, Lemma~\ref{lem:walsh-subsample} gives
\[
\max_{r\in\{0,\dots,m\}}|\alpha_0(r)|\le 1
\qquad\text{and}\qquad
\sum_{\ell\in S}\max_{r\in\{0,\dots,m\}}|c_\ell(r)|
\le \log_2m.
\]

Combining these bounds yields, for a universal constant $C_{\ell_1}$,
\[
    \mathbb{E}[A] \le C_{\ell_1}\log(m+1)\cdot \mathbb{E}[\mathrm{MCerr}_{T}(G)]. \qedhere
\]
\end{proof}

\subsection{Truthfulness Requires Diverse Predictions}
\label{sec:diverse}

We next prepare to lower-bound the noise term
\(\sum_{v\in V_T}|U_v|\) in
Observation~\ref{obs:constant-group-bias-noise}. To preview, the adaptive bucketing theorem we present in the following
Section~\ref{sec:noise-control} will lower-bound this noise term in terms of the prediction-diversity parameter
\[
N=\sum_{v\in V_T}\sqrt{n_v}, \qquad \text{where $n_v:=|\{t\le T:p^t=v\}|$}.
\]
Thus, in this section, we show that the upper bound from Section~\ref{sec:walsh-l1-truthfulness} on the $\ell_1$-truthfulness $A:=\sum_{t\le T}|p^t-x^t|$ forces \(N\) to be large.
We do this by showing that an honest forecaster must use many moderately populated prediction buckets, resulting in \(N\) being large.

To begin, note that the ``honest'' prediction strategy of predicting $\mathbb{E}[y^t] = x^t$ at all rounds obtains high marginal calibration error: our instance has $x^t$ take on many different values, so the honest predictor accumulates noise-driven empirical bias in each of these many ``prediction bins.''

A natural question is whether a forecaster can do better by ``consolidating'' predictions---using fewer distinct prediction values to reduce the number of bins. The key insight of this section is that such consolidation is incompatible with the $\ell_1$-truthfulness constraint. A forecaster that stays close to honest predictions (in $\ell_1$) cannot concentrate its predictions on a small number of values; it must spread them out roughly as the honest forecaster would. Formally, we have:


\begin{lemma}[$\ell_1$-truthfulness implies diverse predictions]
\label{cor:N-from-A}
For any predictor, it holds that:
\[ 
\mathbb{E}[N]\ \ge\ \frac{T}{4 \sqrt{\mathbb{E}[A] + \frac{T}{m} + 1}}.
\]
\end{lemma}

\begin{proof}
Sort the pairs $(x_t,p_t)$ by the mean values, writing
$
x_{(1)} \le \cdots \le x_{(T)}
$
for the nondecreasing rearrangement of the multiset $\{x_t\}_{t\le T}$ and $p_{(1)},\dots,p_{(T)}$ for the correspondingly permuted predictions. Rather than working with the sorted contexts directly, it will be convenient to instead introduce the following uniformly spaced proxy sequence:
\[
z_i := \frac14 + \frac{i-1}{2(T-1)}, \qquad i=1,\dots,T, \quad \text{and let $d := z_{i+1}-z_i = 1/(2(T-1))$.}
\]

\begin{claim}
\label{claim:cl1}
The uniformly spaced proxy sequence is close to the sorted context sequence:
\[
|x_{(i)} - z_i| \le \frac1m \qquad \text{for all } i=1,\dots,T.
\]
\end{claim}

\begin{proof}[Proof of Claim~\ref{claim:cl1}]
    Write $T = qm + r$ with $q := \lfloor T/m\rfloor$ and $0 \le r < m$. Since we pick $m \le T^{1/3}$, we have $q \ge m-1$. In sorted order, each grid point $x_j$ appears
$b_j := q + \mathbf{1}[j \le r]$
times in the block
\[
I_j := \{s_j,\dots,s_j+b_j-1\},
\qquad
s_j := 1 + (j-1)q + \min\{r,j-1\}.
\]
Thus $x_{(i)} = x_j$ for all $i \in I_j$. Since $(z_i)$ is increasing with step $d$ and
$b_j - 1 \le \frac{T-1}{m},$
we thus have for every $i \in I_j$ that
$|z_i - z_{s_j}| \le (b_j-1)d \le \frac{1}{2m},$
and so the claim will follow if we can show $|x_j-z_{s_j}| \le 1/(2m)$. 

To prove this, let $a:=j-1$. Consider first the case $a \le r$. Then, $s_j-1=a(q+1)$, so
\[
x_j-z_{s_j} = \frac{a(q+r-m)}{2(m-1)(T-1)}.
\]
If $r=0$ then $a=0$ hence $x_j-z_{s_j} = 0$; otherwise, $q+r-m \ge 0$ hence
$|x_j-z_{s_j}| \le \frac{aq}{2(m-1)(T-1)} \le \frac{1}{2m}.$

It remains to consider the case $a>r$. In this case, $s_j-1=aq+r$, so
\[
x_j-z_{s_j} = \frac{a(q+r-1)-r(m-1)}{2(m-1)(T-1)}.
\]
The numerator is at least $0$ because
$a(q+r-1)-r(m-1) = (a-r)(q+r-1) + r(q-m+1) \ge 0,$
and it is at most $(m-1)(q-1)$ because
\[
a(q+r-1)-r(m-1) = a(q-1) + r(a-m+1) \le a(q-1) \le (m-1)(q-1).
\]
Therefore,
$|x_j-z_{s_j}| \le \frac{q-1}{2(T-1)} \le \frac{1}{2m}$,
and our claim follows.
\end{proof}

As a direct consequence of Claim~\ref{claim:cl1}, using $|u-v| \ge |u-w| - |w-v|$ with $(u,v,w)=(p_{(i)},x_{(i)},z_i)$ and summing over $i$ gives the following bound on $A$:
\[
A = \sum_{i=1}^T |p_{(i)}-x_{(i)}| \ge \sum_{i=1}^T |p_{(i)}-z_i| - \frac{T}{m}.
\]
Now, for each $v \in V_T$, let
$S_v := \{i : p_{(i)} = v\}, \text{ such that } n_v := |S_v|,$
and write $S_v = \{i_1<\cdots<i_{n_v}\}$. Setting $a_j := z_{i_j}$, we have $a_{j+1}-a_j \ge d$ for all $j$. Hence
\[
\sum_{i\in S_v} |v-z_i| = \sum_{j=1}^{n_v} |v-a_j|
\ge \min_{u\in\mathbb{R}} \sum_{j=1}^{n_v} |u-a_j|.
\]
The minimum is attained at a median. Pairing the left and right halves around a median gives
\[
\min_{u\in\mathbb{R}} \sum_{j=1}^{n_v} |u-a_j|
\ge \sum_{j=1}^{\lfloor n_v/2\rfloor} (a_{n_v+1-j}-a_j)
\ge d\sum_{j=1}^{\lfloor n_v/2\rfloor} (n_v+1-2j)
= d\Big\lfloor \frac{n_v^2}{4} \Big\rfloor
\ge \frac{d}{4}(n_v^2-1).
\]
Summing this over $v$ (and using that $d = 1/(2(T-1)) \ge 1/(4T)$ for $T\ge 2$ and $|V_T|\le T$) yields the following bound on $A$:
\[
A + \frac{T}{m} \ge 
\sum_{i=1}^T |p_{(i)}-z_i|
= \sum_{v\in V_T} \sum_{i\in S_v} |v-z_i|
\ge \frac{d}{4}\sum_{v\in V_T}(n_v^2-1)
\ge \frac{1}{16T}\sum_{v\in V_T} n_v^2 - \frac14.
\]
Rearranging and replacing the constant $1/4$ by $1$ for a better-looking expression, we get
\[
\sum_{v\in V_T} n_v^2 \le 16T\left(A + \frac{T}{m} + 1\right).
\]
To convert this inequality into a lower bound on the predictor's diversity $N$ in terms of $A$, we perform the following algebraic trick. Notice that $\sum_v n_v = T$, and that by Hölder's inequality,
\[
T^3 =
\left(\sum_{v\in V_T} n_v\right)^3
\le \left(\sum_{v\in V_T} \sqrt{n_v}\right)^2 \left(\sum_{v\in V_T} n_v^2\right)
= N^2 \sum_{v\in V_T} n_v^2.
\]
Rearranging and applying the above upper bound on $\sum_{v\in V_T} n_v^2$ then gives
\[
N \ge \frac{T^{3/2}}{\sqrt{\sum_v n_v^2}}
\ge \frac{T}{4\sqrt{A+T/m+1}}.
\]
Finally, $u \mapsto (u+T/m+1)^{-1/2}$ is convex on $[0,\infty)$, so Jensen's inequality yields
\[
\mathbb{E}[N]
\ge \frac{T}{4}\,\mathbb{E}\big[(A+T/m+1)^{-1/2}\big]
\ge \frac{T}{4\sqrt{\mathbb{E}[A]+T/m+1}}. \qedhere
\]
\end{proof}

\subsection{Controlling the Noise Contribution under Adaptive Bucketing}
\label{sec:noise-control}

We now lower-bound the noise term
$\sum_{v\in V_T}|U_v|$
from Observation~\ref{obs:constant-group-bias-noise}. Section~\ref{sec:diverse} shows that small bias
\(A\) forces the diversity parameter \(N=\sum_v\sqrt{n_v}\) to be large. The main result of this
section shows that large \(N\)
forces large bucketed noise --- even if the forecaster tries to prevent this.

The core difficulty is that the forecaster may perform \emph{adaptive noise bucketing}: it observes past noise realizations $Z_1, \ldots, Z_{t-1}$ before choosing which ``bucket'' (prediction value) to place the next noise increment $Z_t$ in. In principle, this adaptivity could allow the forecaster to create cancellations---for example, by attempting to arrange for cancellations of positive noise increments  with negative cumulative sums, and vice versa.

The main technical result of this section (Theorem~\ref{thm:revisit}) shows that no such strategy can substantially reduce the total bucketed noise magnitude. Up to a logarithmic factor, the noise contribution $\sum_v |U_v|$ must be at least as large as $\sum_v \sqrt{n_v}$, where $n_v$ is the number of noise increments routed to bucket $v$. This matches (up to logs) what would happen under non-adaptive routing.

The proof proceeds by analyzing the ``returns to zero'' of each bucket's random walk: each time a bucket's cumulative sum returns to zero, it starts a fresh excursion. Using classical random walk estimates, we show that the expected square-root length of each excursion is $O(\log L)$, which limits how much the forecaster can save by adaptive routing.

\paragraph{Adaptive bucketing}

We now isolate the probabilistic statement we need as an abstract ``noise routing'' problem over a horizon of length $L$.
Let $(Z_t)_{t=1}^L$ be i.i.d.\ with $\mathbb{E}[Z_t]=0$ and $Z_t\in\{\pm h\}$ for some $h\in(0,1]$.
Let $\mathcal{H}$ be any $\sigma$-field independent of $(Z_t)_{t=1}^L$, and suppose a (possibly randomized)
strategy chooses a bucket label $v_t$ on each round $t$ \emph{before} observing $Z_t$, i.e.\ $v_t$ is
measurable with respect to
\[
\mathcal{F}_{t-1} := \mathcal{H} \,\vee\, \sigma(Z_1,\ldots,Z_{t-1}),\qquad t=1,\ldots,L.
\]
Recall that the bucket counts and bucket noise sums are defined as
\[
n_v := |\{t\in\{1,\ldots,L\}: v_t=v\}|,
\qquad
U_v := \sum_{t:\,v_t=v} Z_t.
\]

Now, intuitively, observe that if the buckets $(v_t)_{t=1}^L$ were fixed independently of $(Z_t)_{t=1}^L$, then each $U_v$ would be a length-$n_v$ simple random walk with step size $h$, so typically $|U_v|=\Theta(h\sqrt{n_v})$ and $\sum_v |U_v|=\Theta\!\left(h\sum_v \sqrt{n_v}\right)$. The following key theorem that we prove below shows that even under adaptive bucketing this baseline can be reduced by at most an $O(\log L)$ factor.

\begin{theorem}[Adaptive Noise Bucketing]
\label{thm:revisit}
There exists a universal constant $C_{\mathrm{rev}}>0$ such that for every $L\ge 2$, every $h\in(0,1]$,
and every adaptive bucketing strategy as above (such that every $v_t$ is $\mathcal{F}_{t-1}$-measurable), it holds that
\[
\mathbb{E}\!\left[\sum_v |U_v|\right]
\;\ge\;
\frac{C_{\mathrm{rev}} \, h}{\log(L+1)}\;
\mathbb{E}\!\left[\sum_v \sqrt{n_v}\right].
\]
\end{theorem}

In Corollary~\ref{cor:noise-floor}, we will apply this theorem with $h=1/4$ to the label-noise increments $Z_t:=x^t-y^t$ on the full horizon $t=1,\ldots,T$, with bucket labels given by the realized predictions $v_t:=p^t$.

\subsubsection{Proof of Theorem~\ref{thm:revisit}: Outline}

\noindent
\textbf{Step 1: Reduce to bounding the expected number of returns.}
For each bucket $v$, define its cumulative sum up to each time $t$ as
\[
U_v(t) := \sum_{s\le t: v_s=v} Z_s,\qquad t=0,1,\ldots,L,
\]
and write $U_v := U_v(L)$. Fix $\varepsilon := h/4$, so for any $x\in h\mathbb{Z}$ we have
$|x|\le \varepsilon \iff x=0$. Define
\[
L_\varepsilon := \sum_{t=1}^L \mathbf{1}\!\left[|U_{v_t}(t-1)|\le \varepsilon\right]
= \sum_{t=1}^L \mathbf{1}\!\left[U_{v_t}(t-1)=0\right].
\]

\noindent
In Lemma~\ref{lem:revisit-drift} we show that on rounds with $U_{v_t}(t-1)=0$, the conditional expected increase in $\sum_v |U_v|$ is $h$, while on other rounds it is nonnegative, and so
\[
\mathbb{E} \Big[\sum_v |U_v| \Big] \;\ge\; h\,\mathbb{E}[L_\varepsilon].
\]

\noindent
\textbf{Step 2: Excursions imply $L_\varepsilon$ must be large if $\sum_v\sqrt{n_v}$ is large.}
For a bucket $v$, the evolution of its noise \emph{at the times it is selected} is a simple random walk on $h\mathbb{Z}$ starting at $0$. Each time it returns to $0$
it starts a new \emph{excursion}. Let $R_v$ be the number of excursions for bucket $v$, so that $L_\varepsilon=\sum_v R_v$.
Writing the excursion lengths for bucket $v$ as $\ell^v_1,\ldots,\ell^v_{R_v}$, so that
$n_v=\sum_{j=1}^{R_v}\ell^v_j$, we have
$
\sqrt{n_v} \;\le\; \sum_{j=1}^{R_v} \sqrt{\ell^v_j}.
$
To bound this quantity for each $v$, we prove a truncated return-time bound for a simple random walk: for an
excursion length $\ell$ (truncated at horizon $L$) one has:
\[
\mathbb{E}[\sqrt{\ell}] \;\le\; c \log(L+1) \text{ for a universal constant } c > 0.
\]
Applying this excursion-by-excursion, summing over buckets, and recalling Step~1 then yields
\[
\mathbb{E}\!\left[\sum_v \sqrt{n_v}\right] \;\le\; c  \; \log(L+1)\; \mathbb{E}[L_\varepsilon] \quad \implies \quad
\mathbb{E} \Big[\sum_v |U_v| \Big] \;\ge\; \frac{h}{c \; \log(L+1)} \mathbb{E}\!\left[\sum_v \sqrt{n_v}\right].
\]

\subsubsection{Proof of Theorem~\ref{thm:revisit}: Step 1}

\begin{lemma}[Lower-bounding noise by returns count]
\label{lem:revisit-drift}
Under the setup defined above, we have:
\[
\mathbb{E}\Big[\sum_v |U_v|\Big] \ge h\,\mathbb{E}[L_\varepsilon].
\]
\end{lemma}

\begin{proof}
Let $\Phi_t := \sum_v |U_v(t)|$ for $0\leq t \leq L$, with $\Phi_0=0$ and $\sum_v |U_v|=\Phi_L$. Since only bucket $v_t$ changes at time $t$, the one-round increment satisfies:
\[
\Phi_t-\Phi_{t-1} = \big|U_{v_t}(t-1)+Z_t\big|-\big|U_{v_t}(t-1)\big|.
\]

Condition on $\mathcal{F}_{t-1}$ and abbreviate $b:=U_{v_t}(t-1)$. If $|b|\le \varepsilon$, then
$b=0$ (as $\varepsilon=h/4$ and $b\in h\mathbb{Z}$), so $\Phi_t-\Phi_{t-1}=|Z_t|=h$
deterministically. Otherwise, by convexity of $x\mapsto |x|$ and $\mathbb{E}[Z_t\mid\mathcal{F}_{t-1}]=0$,
\[
\mathbb{E}[\Phi_t-\Phi_{t-1}\mid\mathcal{F}_{t-1}]
=\mathbb{E}[|b+Z_t|\mid\mathcal{F}_{t-1}]-|b|
\ge 0.
\]

Therefore, for every $t$,
\[
\mathbb{E}[\Phi_t-\Phi_{t-1}\mid \mathcal{F}_{t-1}] \ge h\cdot 1\big[|U_{v_t}(t-1)|\le \varepsilon\big].
\]
Taking expectations, summing over $t$, and telescoping yields $\mathbb{E}\Big[\sum_v |U_v|\Big] = \mathbb{E}[\Phi_L]\ge h\,\mathbb{E}[L_\varepsilon]$.
\end{proof}

\subsubsection{Proof of Theorem~\ref{thm:revisit}: Step 2}

Fix a bucket \(v\) with \(n_v \ge 1\). Enumerate its update times as \(1 \le t_{v,1} < \cdots < t_{v,n_v} \le L\), setting \(t_{v,0} := 0\)
for convenience. Define the local-time partial sums \(S_0^v := 0\) and
\(S_k^v := \sum_{i=1}^k Z_{t_{v,i}}\) for \(k = 1,\ldots,n_v\). Thus \(U_v = S_{n_v}^v\).

Define the local ``near-zero-before-update'' indicators
\[
A_k^v := \mathbf{1}\!\left\{|S_{k-1}^v|\le \varepsilon\right\}
= \mathbf{1}\!\left[S_{k-1}^v=0\right],\qquad k=1,\ldots,n_v,
\]
where the equality uses $S_{k-1}^v\in h\mathbb{Z}$ and the choice $\varepsilon=h/4$.

Define the corresponding count $R_v$:
\[
  R_v := \sum_{k=1}^{n_v} A^v_k, \quad \text{so that } L_\varepsilon = \sum_v R_v .
\]
Define renewal indices by \(\kappa_1^v := 1\) and, for \(j \ge 1\),
\[
\kappa_{j+1}^v := \min\{k > \kappa_j^v : A_k^v = 1\},
\]
with the convention \(\kappa_{R_v+1}^v := n_v + 1\). Define excursion lengths
\[
\ell_j^v := \kappa_{j+1}^v - \kappa_j^v, \qquad j = 1,\ldots,R_v.
\]

\begin{lemma}[Subadditivity decomposition]\label{lem:revisit-subadd}
For every bucket $v$,
$\sqrt{n_v}\ \le\ \sum_{j=1}^{R_v} \sqrt{\ell^v_j}.$
Consequently,
\begin{equation*}
  \sum_v \sqrt{n_v}\ \le\ \sum_v\ \sum_{j=1}^{R_v} \sqrt{\ell^v_j}.
\end{equation*}
\end{lemma}

\begin{proof}
By construction, we have $n_v = \sum_{j=1}^{R_v} \ell^v_j.$ Now applying the inequality $\sqrt{a+b}\le \sqrt a+\sqrt b$ for $a,b\ge 0$, iterated over the $R_v$ terms,
and summing over $v$ yields the inequality.
\end{proof}

The following result bounds the expected truncated root-length of a random walk excursion.

\begin{proposition}[Return-time bound for increments]\label{prop:revisit-return}
Let $h\in(0,1]$ and let $(X_k)_{k\ge1}$ be i.i.d.\ with
$\mathbb{P}(X_k=h)=\mathbb{P}(X_k=-h)=1/2$.
Let $S_n:=\sum_{k=1}^n X_k$ with $S_0:=0$ and define the first return time
\[
  \tau_0 := \inf\{n\ge 1 : S_n=0\}.
\]
Then there exists a universal constant $C_{\mathrm{ret}}>0$ such that for all integers $L\ge 2$,
\[
  \mathbb{E}\big[\sqrt{\min(\tau_0,L)}\big]\ \le\ C_{\mathrm{ret}}\log(L+1).
\]
Moreover, the same bound holds conditionally on any past $\sigma$-field $\mathcal{G}$
that is independent of the future increments $(X_k)_{k\ge1}$ (so that, given $\mathcal{G}$,
the increments remain i.i.d.\ with the same $\{\pm h\}$ law).
\end{proposition}

\begin{proof}
Consider the rescaled random walk: let $\tilde X_k := X_k/h\in\{\pm1\}$ and $\tilde S_n:=S_n/h$. Then, $\tau_0=\inf\{n\ge1:\tilde S_n=0\}$, which does not depend on $h$, so it suffices to prove the claim for $h=1$.

Recall the standard fact~\citep[Lemma~2 and Eq.~3.7, p.~78]{fellervol1} that for the simple symmetric walk with steps $\pm1$, its first return time $\tau_0$ is even a.s., and satisfies for $n\ge1$ that:
\[
  \mathbb{P}(\tau_0=2n)\ =\ \frac{1}{2n-1}\binom{2n}{n}4^{-n}.
\]

Using \(\binom{2n}{n} \le 4^n/\sqrt{\pi n}\), we obtain \(\mathbb{P}(\tau_0 = 2n) \le c_0 n^{-3/2}\) for a universal \(c_0>0\), so for all \(m \ge 1\),
\[
\mathbb{P}(\tau_0 > m)
= \sum_{k>\lceil m/2\rceil}\mathbb{P}(\tau_0 = 2k)
\le c_0\sum_{k>\lceil m/2\rceil} k^{-3/2}
\le \frac{c_2}{\sqrt{m}}
\]
for a universal \(c_2>0\). We now have for \(L\ge2\) that
\[
\mathbb{E}[\sqrt{\min(\tau_0,L)}]
= \sum_{m=0}^{L-1}(\sqrt{m+1}-\sqrt{m})\,\mathbb{P}(\min(\tau_0,L)>m)
= \sum_{m=0}^{L-1}(\sqrt{m+1}-\sqrt{m})\,\mathbb{P}(\tau_0>m).
\]
Since \(\sqrt{m+1}-\sqrt{m}\le (2\sqrt{m})^{-1}\) for \(m\ge1\), the above tail bound gives
\[
\mathbb{E}[\sqrt{\min(\tau_0,L)}]
\le 1 + \sum_{m=1}^{L-1}\frac{1}{2\sqrt{m}}\cdot \frac{c_2}{\sqrt{m}}
= 1 + \frac{c_2}{2}\sum_{m=1}^{L-1}\frac{1}{m}
\le C_{\rm ret}\log(L+1).
\]
If \(\mathcal{G}\) is any \(\sigma\)-field independent of the future increments \((X_k)_{k\ge1}\), then conditionally on \(\mathcal{G}\) the increments remain i.i.d. with the same law, so the same bound holds for \(\mathbb{E}[\sqrt{\min(\tau_0,L)}\mid \mathcal{G}]\).
\end{proof}

For the next lemma we append an auxiliary i.i.d.\ continuation $Z_{L+1},Z_{L+2},\dots$ of the noise process and
extend the filtration by
\[
\mathcal{F}_t := \mathcal{H}\,\vee\,\sigma(Z_1,\dots,Z_t),\qquad t\ge 0,
\]
where the added increments remain independent of $\mathcal{H}$ and of $(Z_t)_{t=1}^L$.

\begin{lemma}[Predictable subsequences of i.i.d.\ increments]\label{lem:predictable-subsequence}
Let $(\tau_r)_{r\ge 1}$ be strictly increasing finite random times such that
\[
\{\tau_r=t\}\in \mathcal{F}_{t-1}\qquad\text{for every }r\ge 1,\ t\ge 1.
\]
Then for every $r\ge 1$, conditional on $\mathcal{F}_{\tau_r-1}$, the tail
\[
Z_{\tau_r},\ Z_{\tau_{r+1}},\ Z_{\tau_{r+2}},\ \dots
\]
has the same finite-dimensional distributions as an i.i.d.\ $\{\pm h\}$ sequence. Equivalently, for every
$q\ge 1$ and bounded measurable functions $\varphi_1,\dots,\varphi_q$,
\[
\mathbb{E}\Big[\prod_{j=1}^q \varphi_j(Z_{\tau_{r+j-1}})\,\Bigm|\,\mathcal{F}_{\tau_r-1}\Big]
= \prod_{j=1}^q \mathbb{E}[\varphi_j(Z_1)].
\]
\end{lemma}

\begin{proof}
Fix $r\ge 1$, $q\ge 1$, and bounded measurable functions $\varphi_1,\dots,\varphi_q$.
Write $\mu_j := \mathbb{E}[\varphi_j(Z_1)]$.
We claim that for every bounded $\mathcal{F}_{\tau_r-1}$-measurable random variable $Y$,
\begin{equation}\label{eq:predictable-subsequence-factorization}
\mathbb{E}\Big[Y\prod_{j=1}^q \varphi_j(Z_{\tau_{r+j-1}})\Big]
= \mathbb{E}[Y]\prod_{j=1}^q \mu_j.
\end{equation}
This implies the stated conditional identity.

We prove \eqref{eq:predictable-subsequence-factorization} by induction on $q$.
For $q=1$, using that $Y\mathbf{1}[\tau_r=t]\in \mathcal{F}_{t-1}$ for each $t\ge 1$,
\[
\mathbb{E}[Y\varphi_1(Z_{\tau_r})]
= \sum_{t\ge 1}\mathbb{E}[Y\mathbf{1}[\tau_r=t]\varphi_1(Z_t)]
= \sum_{t\ge 1}\mathbb{E}\big[Y\mathbf{1}[\tau_r=t]\mathbb{E}[\varphi_1(Z_t)\mid \mathcal{F}_{t-1}]\big]
= \mu_1\,\mathbb{E}[Y].
\]

Now assume \eqref{eq:predictable-subsequence-factorization} holds for $q-1$.
Then
\begin{align*}
\mathbb{E}\Big[Y\prod_{j=1}^q \varphi_j(Z_{\tau_{r+j-1}})\Big]
&= \sum_{t\ge 1}\mathbb{E}\Big[
Y\Big(\prod_{j=1}^{q-1} \varphi_j(Z_{\tau_{r+j-1}})\Big)\mathbf{1}[\tau_{r+q-1}=t]\varphi_q(Z_t)\Big].
\end{align*}
On the event $\{\tau_{r+q-1}=t\}$ all earlier selected times $\tau_r,\dots,\tau_{r+q-2}$ are $<t$, so the factor preceding
$\varphi_q(Z_t)$ is $\mathcal{F}_{t-1}$-measurable. Therefore
\begin{align*}
\mathbb{E}\Big[Y\prod_{j=1}^q \varphi_j(Z_{\tau_{r+j-1}})\Big]
&= \sum_{t\ge 1}\mathbb{E}\Big[
Y\Big(\prod_{j=1}^{q-1} \varphi_j(Z_{\tau_{r+j-1}})\Big)\mathbf{1}[\tau_{r+q-1}=t]
\mathbb{E}[\varphi_q(Z_t)\mid \mathcal{F}_{t-1}]\Big] \\
&= \mu_q\,\mathbb{E}\Big[Y\prod_{j=1}^{q-1} \varphi_j(Z_{\tau_{r+j-1}})\Big].
\end{align*}
Applying the induction hypothesis completes the proof.
\end{proof}

Having proved Proposition \ref{prop:revisit-return}, we now relate it back to the cumulative bucket noise in our setup.

\begin{lemma}[Excursions have logarithmic expected root-length]\label{lem:revisit-exc-Olog}
Under the assumptions of Theorem~\ref{thm:revisit} and with the choice of $\varepsilon$ fixed above, there exists a constant $C_{\mathrm{ret}}>0$ such that 
\begin{equation}\label{eq:revisit-denom-Leta}
  \mathbb{E}\left[\sum_v \sqrt{n_v}\right]\ \le\ C_{\mathrm{ret}}\log(L+1)\ \mathbb{E}[L_\varepsilon].
\end{equation}
\end{lemma}

\begin{proof}
For each bucket $v$, extend its update-time sequence beyond the horizon by setting
\[
t_{v,k} := L + (k-n_v), \qquad k>n_v.
\]
For $k\le n_v$, the event $\{t_{v,k}=t\}$ is $\mathcal{F}_{t-1}$-measurable because it is determined by the history
of which rounds up to $t$ selected bucket $v$. For $k>n_v$, the event $\{t_{v,k}=t\}$ can only occur for $t>L$ and is then
\[
\{t_{v,k}=t\} = \{n_v = L+k-t\},
\]
which belongs to $\mathcal{F}_{L-1}\subseteq \mathcal{F}_{t-1}$ because the entire selection sequence up to time $L$ is known by time $L-1$.
Hence the infinite sequence $(t_{v,k})_{k\ge 1}$ is predictable.

We will first prove an intermediate claim, which implies the result of the lemma: for every bucket $v$ and every excursion index $j\in\{1,\dots,R_v\}$, we will now show that
\begin{equation}
    \label{eq:bucketindex_ret}\mathbb{E}\big[\sqrt{\ell^v_j}\,\big|\,\mathcal{F}_{t_{v,\kappa_j^v}-1}\big]\ \le\ C_{\mathrm{ret}}\log(L+1).
\end{equation}
Towards this, fix $v,j$ and set
\[
\sigma_{v,j} := t_{v,\kappa_j^v},\qquad
M_{v,j} := n_v-\kappa_j^v+1,\qquad
\mathcal{G} := \mathcal{F}_{\sigma_{v,j}-1}.
\]
By definition of $\kappa_j^v$ we have $|S^v_{\kappa_j^v-1}|\le \varepsilon$, and by the standing lattice assumption this implies
$S^v_{\kappa_j^v-1}=0$. The event $\{\sigma_{v,j}=t\}$ is $\mathcal{F}_{t-1}$-measurable, because by time $t-1$ one knows whether
round $t$ is the $j$-th selection of bucket $v$ whose pre-update bucket sum is $0$. Hence $\sigma_{v,j}$ is predictable.

Now define
\[
X_r := Z_{t_{v,\kappa_j^v+r-1}},\qquad r=1,2,\dots.
\]
For each fixed $k\ge 1$, Lemma~\ref{lem:predictable-subsequence} applied to the predictable time sequence
$(t_{v,r})_{r\ge 1}$ shows that conditional on $\mathcal{F}_{t_{v,k}-1}$, the tail
\[
Z_{t_{v,k}},\ Z_{t_{v,k+1}},\ Z_{t_{v,k+2}},\ \dots
\]
is i.i.d.\ with law $\{\pm h\}$. Moreover, the event $\{\kappa_j^v=k\}$ is $\mathcal{F}_{t_{v,k}-1}$-measurable,
since by time $t_{v,k}-1$ one knows which of the first $k$ updates to bucket $v$ started from zero. Because
$\sigma_{v,j}=t_{v,k}$ on $\{\kappa_j^v=k\}$, summing over $k$ yields that conditional on $\mathcal{G}$ the sequence
$(X_r)_{r\ge 1}$ is i.i.d.\ with $\mathbb{P}(X_r=h)=\mathbb{P}(X_r=-h)=1/2$.
Define the local random walk by
\[
T_0:=0,\qquad T_n:=\sum_{r=1}^n X_r,
\]
and let
\[
\tau_{v,j}^0 := \inf\{n\ge 1:T_n=0\}.
\]
The first $M_{v,j}$ terms of $(X_r)$ are exactly the actual future increments routed to bucket $v$ from the start of the $j$-th excursion
until the horizon ends. Therefore the excursion length is
\[
\ell_j^v = \min(\tau_{v,j}^0,M_{v,j}),
\]
since the excursion ends either at the first return to $0$ or when bucket $v$ is never selected again before time $L$.
Because $M_{v,j}\le L$, we obtain
\[
\ell_j^v \le \min(\tau_{v,j}^0,L).
\]

Conditional on $\mathcal{G}$, the process $(T_n)_{n\ge0}$ satisfies the assumptions of Proposition~\ref{prop:revisit-return}. Hence, by Proposition~\ref{prop:revisit-return}, for a universal constant $C_{\mathrm{ret}}$ we have:
\[
\mathbb{E}\big[\sqrt{\ell^v_j}\,\big|\,\mathcal{F}_{t_{v,\kappa_j^v}-1}\big] \;\le\; \mathbb{E}\big[\sqrt{\min(\tau_{v,j}^0,L)}\,\big|\,\mathcal{F}_{t_{v,\kappa_j^v}-1}\big]
\;\le\; C_{\mathrm{ret}}\log(L+1),
\] 
and our intermediate claim~\eqref{eq:bucketindex_ret} follows. 

To now establish \eqref{eq:revisit-denom-Leta}, we start from Lemma~\ref{lem:revisit-subadd}:
\[
\mathbb{E}\Big[\sum_v \sqrt{n_v}\Big]\le \mathbb{E}\Big[\sum_v\sum_{j=1}^{R_v}\sqrt{\ell_j^v}\Big].
\]

Since $R_v \le n_v \le L$, we may write
\[
\mathbb{E}\left[\sum_{j=1}^{R_v}\sqrt{\ell_j^v}\right]
=
\sum_{j=1}^{L}\mathbb{E}\left[\mathbf{1}[j \le R_v]\sqrt{\ell_j^v}\right]
=
\sum_{j=1}^{L}\mathbb{E}\left[\mathbf{1}[j \le R_v]\mathbb{E}\left[\sqrt{\ell_j^v} \mid \mathcal{F}_{t_{v,\kappa_j^v}-1}\right]\right]
\le C_{\rm ret}\log(L+1)\mathbb{E}[R_v],
\]
where the inner conditional expectation is invoked only on the event $\{j \le R_v\}$ (on which $\kappa_j^v$ and $\ell_j^v$ are defined), and the inequality uses Equation~\eqref{eq:bucketindex_ret} and $\sum_{j=1}^L \mathbf{1}[j \le R_v] = R_v$.

Summing the last display over $v$ and recalling that $\sum_v R_v=L_\varepsilon$ concludes the proof.
\end{proof}

\medskip

\subsubsection{Finishing the Proof of Theorem~\ref{thm:revisit}}

Combining Lemma~\ref{lem:revisit-drift} and Lemma~\ref{lem:revisit-exc-Olog}, we obtain:
\[
  \mathbb{E}\Big[\sum_v \sqrt{n_v}\Big]
  \ \le\ C_{\mathrm{ret}} \log(L+1)\ \mathbb{E}[L_\varepsilon]
  \ \le\ \frac{C_{\mathrm{ret}}}{h}\,\log(L+1)\ \mathbb{E}\Big[\sum_v |U_v|\Big].
\]
Setting $C_{\mathrm{rev}}:= 1 / C_{\mathrm{ret}}$ and rearranging, we obtain the claimed bound of Theorem~\ref{thm:revisit}.

\subsubsection{A Noise Lower Bound via Theorem~\ref{thm:revisit}}


\begin{corollary}[Noise lower bound for the constant-group buckets]
\label{cor:noise-floor}
For \(U_v\) and \(N\) defined in Section~\ref{sec:bias-noise}, there exists a universal constant \(c_2>0\) such that
\[
\mathbb{E}\!\left[\sum_{v\in V_T}|U_v|\right]
\ge
\frac{c_2}{\log(T+1)}\mathbb{E}[N].
\]
\end{corollary}
\begin{proof}
Apply Theorem~\ref{thm:revisit} with horizon $L=T$, increments $Z_t:=x^t-y^t$, step size $h=1/4$, and bucket labels $v_t:=p^t$.
Under the environment~\eqref{eq:env-walsh}, we have $Z_t=-\xi^t/4$, so $(Z_t)_{t=1}^{T}$ are i.i.d., mean $0$, and take values in $\{\pm h\}$.
To verify predictability, let
\[
\mathcal{H} := \sigma\big(x^1,\ldots,x^{T}, \text{algorithm’s internal randomness}\big).
\]
Then $\mathcal{H}$ is independent of $(Z_t)_{t=1}^{T}$, and for each $t$ the prediction $p^t$ is measurable with respect to the past transcript $\sigma(x^1,\ldots,x^{t},y^1,\ldots,y^{t-1})$.
Since $y^s = x^s - Z_s$ for all $s<t$ and $(x^s)_{s\le T}$ are $\mathcal{H}$-measurable, it follows that $p^t$ is measurable with respect to $\mathcal{F}_{t-1}:=\mathcal{H}\vee\sigma(Z_1,\ldots,Z_{t-1})$.

Therefore Theorem~\ref{thm:revisit} gives, with $c_2:=C_{\mathrm{rev}}/4$,
\[
\mathbb{E}\Big[\sum_{v\in\mathcal{V}_{T}} |U_v|\Big]
\ge \frac{C_{\mathrm{rev}}}{4\,\log(T+1)}\cdot \mathbb{E}\Big[\sum_{v\in\mathcal{V}_{T}} \sqrt{n_v}\Big]
= \frac{c_2}{\log(T+1)}\,\mathbb{E}[N]. \qedhere
\]
\end{proof}

\subsection{Putting the Bias and Noise Bounds Together}
\label{sec:final-proof}

We now finish the proof of Theorem~\ref{thm:walsh-localrate-nonneg}, by combining the constant-group bias-noise decomposition with our bounds on bias/truthfulness, on the prediction diversity, and on adaptive-bucketing noise.

For brevity, let us write
$
\mathrm{MC}:=\mathbb{E}\!\left[\operatorname{MCerr}_T(\mathcal G)\right].$
Since \(g_{\rm all}\in\mathcal G\), we have
$\operatorname{MCerr}_T(\mathcal G)\ge \operatorname{Err}_T(g_{\rm all}).$
Therefore, taking expectations in Observation~\ref{obs:constant-group-bias-noise} and applying
Corollary~\ref{cor:noise-floor} gives the lower bound:
\begin{equation}
\label{eq:mc-basic-lb}
\mathrm{MC} \;\ge\; \frac{c_2}{\log(T+1)}\,\mathbb{E}[N] - \mathbb{E}[A].
\end{equation}
%
Now recall that by Lemma~\ref{lem:l1-truthfulness}, we have the bias/truthfulness bound:
\begin{equation}
\label{eq:A-truthfulness-final}
\mathbb{E}[A] \le C_{\ell_1}\log(m+1) \cdot \mathrm{MC},
\end{equation}
and by Lemma~\ref{cor:N-from-A}, we have the prediction diversity bound:
\begin{equation}
\label{eq:N-from-A-final}
\mathbb{E}[N]\ge \frac{T}{4\sqrt{\mathbb{E}[A] + \frac{T}{m} + 1}}.
\end{equation}
Set the shorthand $L:=\log(T+1)$ and note $\log(m+1)\le L$. Combining \eqref{eq:mc-basic-lb} and \eqref{eq:A-truthfulness-final} implies
\[
(1+C_{\ell_1}L) \cdot \mathrm{MC} \ge \frac{c_2}{L}\,\mathbb{E}[N] \implies \mathrm{MC} \ge \frac{c_3}{L^2}\,\mathbb{E}[N] \text{ for some } c_3 > 0.
\]
Combining this with \eqref{eq:N-from-A-final} and again using \eqref{eq:A-truthfulness-final}, we thus obtain the following bound:
\[
\mathrm{MC} \ge \frac{c_3T}{4L^2\sqrt{C_{\ell_1}L \cdot \mathrm{MC} +T/m+1}}.
\]
Now consider the two cases for which term dominates in the denominator of this bound.

If $C_{\ell_1}L \cdot \mathrm{MC} \le T/m+1$, then
$\mathrm{MC} \ge \frac{c_3T}{4L^2\sqrt{2(T/m+1)}}.$
For all sufficiently large $T$, we have
$\frac{1}{2}T^{1/3}\le m\le T^{1/3},$
hence $T/m+1\le 3T^{2/3}$, and thus
$\mathrm{MC} \ge c_4\,\frac{T^{2/3}}{L^2}$
for some $c_4>0$.

If instead $C_{\ell_1}L \cdot \mathrm{MC} >T/m+1$, then
$\mathrm{MC} \ge \frac{c_3T}{4L^2\sqrt{2C_{\ell_1}L \cdot \mathrm{MC}}},$
so for some $c_5>0$,
$\mathrm{MC}^{3/2} \ge c_5\,\frac{T}{L^{5/2}}$ and thus $\mathrm{MC} \ge c_5\,\frac{T^{2/3}}{L^{5/3}} \ge c_5\,\frac{T^{2/3}}{L^2}.$

In both cases,
$
\mathrm{MC} \ge c\,\frac{T^{2/3}}{\log^2(T+1)}
$
for some $c>0$ and $T$ large enough, proving Theorem~\ref{thm:walsh-localrate-nonneg}.

\section{Discussion}

We have established tight $\Theta(T^{2/3})$ bounds (up to logarithmic factors) for online multicalibration, separating it from marginal calibration. In the case of prediction independent groups, there is a narrow regime of group sizes for which minimax optimal rates remain open. \cite{noarov2023high} gives $\widetilde O(T^{2/3})$ multicalibration rates for all group families of size polynomial in $T$. On the other hand, we have established that this is tight already for group families scaling only as $\log^3(T)$, separating the minimax rates for multicalibration from those of marginal calibration. We have also observed that no separation is possible for constant sized group families (i.e.\ those not scaling with $T$ at all). The regime in which the group family grows with $T$ but at a rate of $o(\log^3 T)$ remains open. In Appendix \ref{sec:oracle-lb} we give an oracle lower bound that presents a barrier to giving a rate preserving reduction from multicalibration to marginal calibration for a group family of size only $O(\log T)$, but we leave tightly understanding minimax rates for multicalibration in this small-but-super-constant group size regime open.

\subsection*{Acknowledgments} We gratefully acknowledge funding from the Simons Foundation, the UK AI Safety Institute, and the NSF ENCoRE TRIPODS Institute. 

We thank Princewill Okoroafor for enlightening discussions at an early stage of this work.

The authors used AI tools, specifically GPT 5.1 and 5.4 Pro, and GPT 5.4 in the Codex environment in the development of this paper; all of the final theorems and proofs are written and verified by the authors, and all of  the exposition and discussion of related work was written without AI assistance.  
\bibliographystyle{alpha}
\bibliography{bib}

\appendix
\section{Constant-Sized Families of Binary Prediction-Independent Groups}
\label{sec:pred-independent}

The lower bound in Theorem~\ref{thm:main} crucially exploits the fact that the groups $g_1,g_2,g_3$
are allowed to depend on the prediction value $v$. In this section we show that an analogous
lower bound cannot hold for a constant number (i.e. not growing with $T$) of binary \emph{prediction-independent} groups---groups
that depend only on the context $x$ and not on the prediction $v$.

To do this we show a simple reduction from the problem of sequential adversarial multicalibration (for prediction independent groups) to the problem of sequential adversarial \emph{marginal} calibration, with rates that degrade at most exponentially with the number of groups to be multicalibrated. An implication of this is that for any constant number of groups, the best rate for multicalibration is the same (up to constant factors depending on the number of groups) as it is for marginal calibration --- and in particular, by the recent result of \cite{dagan2025breaking}, $O(T^{2/3-\epsilon})$. 

The idea is extremely simple: if we have $k$ prediction independent binary groups, then before we make our prediction, we know which of the at most $2^k$ \emph{combinations of groups} are active at each round $t$ before we must make our prediction. The sequence of rounds on which each combination of groups is active is by construction disjoint from every other. Thus we can instantiate a separate copy of any marginal calibration algorithm for each of the $2^k$ possible combinations of groups and run each on the corresponding subsequence. Naively this results in a blow-up in rates of $2^k$. In this section we give a somewhat more refined bound that depends for each group on the number of distinct combinations of groups that it participates in, and further takes advantage of the convexity of calibration error upper bounds.  

\begin{theorem}[Structure-aware binary prediction-independent multicalibration]
\label{thm:pred-indep}
Let $k\in\mathbb{N}$ and let $G=\{g_1,\dots,g_k\}$ be a family of binary prediction-independent groups.
Suppose there exists an online prediction algorithm $A^{\mathrm{marg}}$ and a function $R:\mathbb{N}\to[0,\infty)$
that is nondecreasing and satisfies $R(0)=0$, such that for every horizon $n\ge 1$ and every sequence
$(x^t,y^t)_{t=1}^n$, when $A^{\mathrm{marg}}$ is run for $n$ rounds we have:
\[
  \mathbb{E}\Big[\sum_{v\in V_n} \Big|\sum_{t=1}^n \mathbf{1}[p^t=v]\,(p^t-y^t)\Big|\Big]
  \;\le\; R(n).
\]
 where the expectation is
 over the internal randomness of $A^{\mathrm{marg}}$.
Assume moreover that $R$ extends to a concave, nondecreasing function on $[0,\infty)$.
Then there exists an online algorithm $A^{\mathrm{multi}}$ such that for every horizon $T\ge 0$ and every
sequence $(x^t,y^t)_{t=1}^T$, the following holds.

Partition the rounds by membership patterns: for $x\in X$, let $z(x)\in\{0,1\}^k$ with $z_j(x)=g_j(x)$.
For each realized pattern $z$, define the cell $C_z:=\{t: z(x^t)=z\}$ and its size $T_z:=|C_z|$.
For each $j\in\{1,\dots,k\}$ set
\[
  T_j := \sum_{z: z_j=1} T_z, \qquad K_j := \big|\{ z : z_j=1,\ T_z>0\}\big|.
\]
 Then, for every $j$, we have
\begin{equation}\label{eq:struct-basic}
  \mathbb{E}[\mathrm{Err}_T(g_j)] \;\le\; \sum_{z: z_j=1} R(T_z)
  \;\le\; K_j\, R(T_j/K_j) \quad\text{if $R$ is concave},
\end{equation}
with the convention that $K_j=0$ implies both sides are $0$ (since $T_j=0$ and $R(0)=0$).
In addition, the multicalibration error satisfies the following valid bounds:
\begin{equation}\label{eq:mc-cells}
  \mathbb{E}[\mathrm{MCerr}_T(G)] \;\le\; \sum_{z\in\{0,1\}^k} R(T_z)
  \;\le\; 2^k\, R\big(T/2^k\big)\quad\text{if $R$ is concave}.
\end{equation}
\end{theorem}


\begin{proof}
Fix a sequence $(x^t,y^t)_{t=1}^T$. Then every cell $C_z$ and every count $T_z$, $T_j$, and $K_j$ defined in the
statement is deterministic.

For each pattern $z\in\{0,1\}^k$, create an independent copy $A^z$ of $A^{\mathrm{marg}}$.
On each round $t$, compute $z_t:=z(x^t)$, query $A^{z_t}$ for a prediction $p^t$, then reveal $y^t$ and
feed $(x^t,y^t,p^t)$ only to the copy $A^{z_t}$. This defines the algorithm $A^{\mathrm{multi}}$.

For each $z$ and $v\in[0,1]$, define the cell-wise bias
\[
  B_z(v) := \sum_{t\in C_z} \mathbf{1}[p^t=v]\,(p^t-y^t).
\]
Write $V(C_z)$ for the set of prediction values output by $A^z$ on its own rounds $C_z$.
By construction, on the rounds $C_z$ the copy $A^z$ runs exactly as $A^{\mathrm{marg}}$ would on the
length-$T_z$ deterministic sequence $((x^t,y^t))_{t\in C_z}$. Therefore the marginal guarantee gives
\begin{equation}\label{eq:cell-R}
  \mathbb{E}\Big[\sum_{v\in V(C_z)} |B_z(v)|\Big] \;\le\; R(T_z).
\end{equation}

Fix $j\in\{1,\dots,k\}$. Since $g_j(x)=1$ if and only if $z_j(x)=1$, we have for each $v$,
\[
  B_T(v,g_j)
  = \sum_{t=1}^T \mathbf{1}[p^t=v]\,g_j(x^t)\,(p^t-y^t)
  = \sum_{z: z_j=1} \sum_{t\in C_z} \mathbf{1}[p^t=v]\,(p^t-y^t)
  = \sum_{z: z_j=1} B_z(v).
\]
By the triangle inequality and summing over $v$, we obtain
\[
  \mathrm{Err}_T(g_j) = \sum_{v\in V_T} |B_T(v,g_j)|
  \;\le\; \sum_{z: z_j=1} \sum_{v\in V(C_z)} |B_z(v)|.
\]
Taking expectations and applying \eqref{eq:cell-R},
\[
  \mathbb{E}[\mathrm{Err}_T(g_j)] \;\le\; \sum_{z: z_j=1} R(T_z),
\]
which proves the first inequality in \eqref{eq:struct-basic}.

For the second inequality in \eqref{eq:struct-basic}, fix a realization of $(T_z)_{z: z_j=1}$ and let
$K=K_j$ and $a_i:=T_{z_i}$ for the $K$ patterns with $z_j=1$. Since $R$ is concave and nondecreasing on
$[0,\infty)$ with $R(0)=0$, Jensen's inequality gives
\[
  \frac{1}{K}\sum_{i=1}^K R(a_i) \;\le\; R\Big(\frac{1}{K}\sum_{i=1}^K a_i\Big)
  = R(T_j/K).
\]
Multiplying by $K$ yields the pathwise bound $\sum_{z: z_j=1} R(T_z) \le K_j\,R(T_j/K_j)$. Taking expectations
establishes the second inequality in \eqref{eq:struct-basic}.

 To bound the multicalibration error, define $A_z := \sum_{v\in V(C_z)} |B_z(v)|\ge 0$ for each pattern $z$.
 For each $j$ we showed pathwise that $\mathrm{Err}_T(g_j) \le \sum_{z: z_j=1} A_z$, hence
 \[
   \mathrm{MCerr}_T(G) \,=\, \max_{1\le j\le k} \mathrm{Err}_T(g_j)
   \;\le\; \max_{1\le j\le k} \sum_{z: z_j=1} A_z
   \;\le\; \sum_{z\in\{0,1\}^k} A_z,
 \]
 where the last inequality holds because each $A_z\ge 0$ and the maximum of partial sums is at most the total sum.
 Taking expectations and using \eqref{eq:cell-R} yields
 \[
   \mathbb{E}[\mathrm{MCerr}_T(G)] \;\le\; \sum_{z\in\{0,1\}^k} \mathbb{E}[A_z]
   \;\le\; \sum_{z\in\{0,1\}^k} R(T_z).
 \]
 Since $R$ is concave and nondecreasing with $R(0)=0$ and $\sum_z T_z = T$, Jensen's inequality gives the pathwise bound
 \[
   \sum_{z\in\{0,1\}^k} R(T_z) \;\le\; 2^k\, R\Big(\frac{1}{2^k}\sum_{z\in\{0,1\}^k} T_z\Big)
   = 2^k\,R(T/2^k),
 \]
 establishing \eqref{eq:mc-cells}.
\end{proof}
  
The theorem shows that the difficulty of prediction-independent multicalibration is governed by the
intersection structure of the groups, quantified by $K_j$ and $T_j$. In particular, when $|G|$ is fixed independently of $T$,
the crude bound $K_j\le 2^{k-1}$ recovers the same rate as the marginal algorithm.

This simple reduction crucially relies on the fact that the groups are \emph{binary} and \emph{prediction independent}, since it needs to identify which subset of groups are active \emph{before} it decides which algorithm will be assigned to make a prediction each day. It establishes that marginal calibration rates are the same as multicalibration rates (as a function of $T$) for any collection of groups whose cardinality is independent of $T$ --- and hence, by the result of \cite{dagan2025breaking}, $o(T^{2/3})$. Contrast this with Theorem \ref{thm:main} which establishes an $\Omega(T^{2/3})$ \emph{lower bound} for multicalibration of even $3$ groups when the groups can depend on the prediction.

\section{An Oracle Lower Bound for Better Black-Box Reductions}
\label{sec:oracle-lb}

This section formalizes a natural class of \emph{proper} black-box reductions from multicalibration to marginal calibration, and proves an oracle lower bound showing that such reductions require
\emph{exponentially many} oracle copies (in the number of groups) in the worst case, even for prediction independent groups --- showing that our reduction in Appendix \ref{sec:pred-independent} is in a sense tight.
Like the lower bounds in Sections~\ref{sec:prediction-dependent} and~\ref{sec:walsh-localrate-nonneg}, the environment here is completely
oblivious, but our instance here is even more benign: the labels are now deterministic given the context; the lower bound follows purely from the
properness constraint and the ``context-blindness'' of marginal calibration algorithms. We show that any ``proper''  black-box reduction from marginal calibration to
multicalibration must itself incur multicalibration error $\Omega(T^{1-\gamma})$ for any constant $\gamma > 0$ unless it uses exponentially many oracle
copies. Our construction uses a binary group family of cardinality $|G| = \Theta(\log T)$, and so serves as an obstruction to a reductions-based strategy to giving $o(T^{2/3})$ upper bounds for multicalibration in this regime. 

\subsection{Context-blind oracles and proper reductions}
 We define the notion of a ``context-blind'' oracle:

\begin{definition}[Context-blind oracle]
\label{def:cb-oracle}
A (possibly randomized) forecasting algorithm $A$ is \emph{context-blind} if for every round $t$ its output
distribution depends only on its internal state (i.e.\ on its own past transcript) and not on the current
context $x^t$. Equivalently, for any realized internal randomness of $A$, the mapping $x^t\mapsto Q^t$ produced
by $A$ on round $t$ is constant.
\end{definition}

\begin{remark}[Context-Blinding]
\label{rem:cb-blindification}
If $A$ is any algorithm with a worst-case marginal calibration guarantee that holds for \emph{all} context
sequences, then feeding $A$ a fixed dummy context $\bar x$ on every round
(and otherwise running it unchanged) produces a context-blind algorithm with the same marginal guarantee.
Thus, any black-box reduction that claims to work for \emph{every} marginal calibration algorithm must in
particular work for some context-blind marginal oracle. In what follows we fix such a context-blind oracle $A$
and treat it purely as a black box satisfying a marginal calibration guarantee.
\end{remark}

We define the family of reductions that our barrier result applies to below. Informally, it corresponds to algorithms that can run $m$ copies of some marginal calibration algorithm $A$, update each algorithm in arbitrary ways, potentially differently for each copy of $A$, and then use prediction distributions that are somewhere in the convex hull of the predictive distributions proposed by each copy of $A$, where the weights of the convex mixture can depend both on context and history. This is e.g. the form that reductions from multigroup regret to marginal (external) regret via sleeping experts constructions take \citep{blum2020advancing,acharyaoracle}. These reductions run one copy of the oracle for each group --- i.e. setting $m = \Theta(|G|)$ and update each oracle for marginal regret on the subsequence corresponding to the rounds at which the corresponding group is active. Our barrier will rule out any similar reduction  obtaining sublinear multicalibration error.

\begin{definition}[Proper $m$-copy black-box reduction]
\label{def:proper-mcopy}
Fix an integer $m\ge 1$. A \emph{proper $m$-copy black-box reduction} $B$ is a meta-algorithm with oracle access
to a context-blind forecasting algorithm $A$, and it is allowed to run $m$ independent copies
$A^{(1)},\dots,A^{(m)}$ of $A$.

On each round $t$:
\begin{enumerate}
\item The context $x^t$ is revealed to $B$.
\item Each copy $A^{(i)}$ outputs a distribution $Q_i^t$ on $[0,1]$.
      Because $A$ is context-blind, each $Q_i^t$ is a (possibly randomized) function only of the past transcript
      of copy $i$, and does not depend on $x^t$.
\item The reduction outputs a distribution $P^t$ satisfying the \emph{properness} constraint
      \[
      P^t \in \mathrm{conv}\{Q_1^t,\dots,Q_m^t\},
      \qquad\text{i.e.}\qquad
      P^t = \sum_{i=1}^m \alpha_{i,t} Q_i^t
      \ \text{ for some }\ \alpha_t\in\Delta_m,
      \]
      where the weights $\alpha_t$ may depend on $x^t$ and all past history.
\item The outcome $y^t$ is revealed.
\item The reduction may choose, for each copy $i$, whether and how to update the state of
      copy $i$ using information available up to this point. (Our lower bound will not depend on any particular
      update scheme; it uses only the context-blindness of the $Q_i^t$ at the moment they are produced and the
      properness constraint on $P^t$.)
\item Finally, the  prediction $p^t\sim P^t$ is drawn.
\end{enumerate}
\end{definition}

\subsection{Hard instance and a logarithmic-size group family}

Fix integers $T\ge 1$ and $k\ge 1$, and set $N:=2^k$.
Let the context space be the $k$-bit hypercube
\[
X:=\{0,1\}^k,
\]
and identify each $x=(x_1,\dots,x_k)\in X$ with the integer
\[
\mathrm{val}(x)\ :=\ \sum_{r=1}^k x_r\,2^{k-r}\ \in\ \{0,1,\dots,N-1\}.
\]
Define the partition of $[0,1]$ into $N$ intervals
\[
J_b :=
\begin{cases}
\big[\frac{b}{N},\frac{b+1}{N}\big),& b=0,1,\dots,N-2,\\[3pt]
\big[\frac{N-1}{N},1\big],& b=N-1.
\end{cases}
\]
Define the mean map $\mu:X\to(0,1)$ by
\[
\mu(x)\ :=\ \frac{\mathrm{val}(x)+\tfrac12}{N}.
\]
Note that $\mu(x)$ is the midpoint of $J_{\mathrm{val}(x)}$.

\paragraph{Distribution over contexts and labels.}
Let $\mathcal{D}_{T,N}$ be the oblivious distribution over $(x^t,y^t)_{t=1}^T$ defined by contexts $x^1,\dots,x^T$ that are i.i.d.\ uniform on $X$ and labels that are deterministicly $y^t := \mu(x^t)$ for each $t$. Thus $(x^t,y^t)$ are independent of the forecaster, and $y^t\in(0,1)$ always.

\paragraph{Group family.}
We use  $k+1=\log_2 N+1$ binary prediction-independent groups:
\[
g_0(x):=1,
\qquad
g_r(x):=x_r\in\{0,1\}\quad(r=1,\dots,k).
\]
Let
\[
G_{\mathrm{bits}} := \{g_0,g_1,\dots,g_k\},
\qquad |G_{\mathrm{bits}}|=k+1.
\]

\subsection{Main theorem}

\begin{theorem}[Oracle lower bound with $|G|=\Theta(\log N)$]
\label{thm:oracle-bits}
Fix integers $T\ge 1$, $k\ge 1$, and set $N:=2^k$.
Let $B$ be any proper $m$-copy black-box reduction (Definition~\ref{def:proper-mcopy}),
and let $A$ be any context-blind oracle (Definition~\ref{def:cb-oracle}).
Run the induced forecaster $B^A$ for $T$ rounds on $\mathcal{D}_{T,N}$ and evaluate multicalibration error
with respect to $G_{\mathrm{bits}}$.
Then
\[
\mathbb{E}\big[\mathrm{MCerr}_T(G_{\mathrm{bits}})\big]
\ \ge\
\frac{1}{8}\left(1-\frac{m}{N}\right)\frac{T}{N^2}.
\]
In particular, if $m\le N/2$ then
\[
\mathbb{E}\big[\mathrm{MCerr}_T(G_{\mathrm{bits}})\big]\ \ge\ \frac{T}{16N^2}.
\]

Equivalently, fix any constant $\gamma\in(0,\tfrac12)$ and let
\[
  k := \big\lfloor \gamma \log_2 T \big\rfloor,\qquad N := 2^k,
\]
so that $|G_{\mathrm{bits}}| = k+1 = \Theta(\log T)$. Then there exists $T_0=T_0(\gamma)$ such that for all
$T\ge T_0$ and all $m=\mathrm{poly}(|G_{\mathrm{bits}}|)$ we have
\[
\mathbb{E}\big[\mathrm{MCerr}_T(G_{\mathrm{bits}})\big]
\ \ge\ \frac{1}{16}\,T^{1-2\gamma}.
\]
\end{theorem}

The proof is a combination of three ingredients:
(i) a properness lemma implying the reduction cannot put much probability mass on the correct interval
$J_{\mathrm{val}(x^t)}$ on average;
(ii) the fact that missing the correct interval implies a squared-loss penalty of order $1/N^2$; and
(iii) a deterministic inequality 
showing that squared loss is controlled by calibration error for the bit groups.

\begin{remark}
The parameter $\gamma$ in Theorem~\ref{thm:oracle-bits} can be chosen arbitrarily small, so the exponent
$1-2\gamma$ can be made arbitrarily close to $1$ while still working with only $|G_{\mathrm{bits}}|=\Theta(\log T)$
binary prediction-independent groups and using only $m=\mathrm{poly}(|G_{\mathrm{bits}}|)$ oracle copies.
\end{remark}

\subsection{Proof of Theorem~\ref{thm:oracle-bits}}

\begin{lemma}[Correct-interval mass bound]
\label{lem:oracle-bits-correctmass}
Let $B^A$ be any proper $m$-copy reduction with a context-blind oracle $A$, run on $\mathcal{D}_{T,N}$.
For each round $t$, let $P^t$ be the external distribution output by $B$, and define
\[
\pi_t\ :=\ P^t\big(J_{\mathrm{val}(x^t)}\big)\ \in\ [0,1].
\]
Then for every $t$,
\[
\mathbb{E}[\pi_t]\ \le\ \frac{m}{N}.
\]
Consequently,
\[
\mathbb{E}\left[\sum_{t=1}^T \mathbf{1}\big[p^t\in J_{\mathrm{val}(x^t)}\big]\right]
=
\mathbb{E}\left[\sum_{t=1}^T \pi_t\right]
\ \le\ \frac{m}{N}\,T.
\]
\end{lemma}
\begin{proof}
Fix a round $t$ and condition on the full transcript up to the end of round $t-1$, including all internal
randomness of $B$ and of all oracle copies. Under $\mathcal{D}_{T,N}$, the fresh context $x^t$ is uniform on
$X$ and independent of this past transcript.

Because the oracle $A$ is context-blind, each copy's distribution $Q_i^t$ depends only on the past transcript
of copy $i$. Hence, under the conditioning, the distributions $Q_1^t,\dots,Q_m^t$ are fixed probability measures
on $[0,1]$ and do not depend on the random draw of $x^t$.

Since $B$ is proper, it outputs
\[
P^t = \sum_{i=1}^m \alpha_{i,t}\,Q_i^t
\qquad\text{for some }\alpha_t\in\Delta_m
\]
(possibly chosen as a function of $x^t$).
For any $b\in\{0,\dots,N-1\}$ we have
\[
P^t(J_b)=\sum_{i=1}^m \alpha_{i,t}\,Q_i^t(J_b)\ \le\ \max_{1\le i\le m} Q_i^t(J_b).
\]
Averaging over the uniform $x^t$ (equivalently $b=\mathrm{val}(x^t)$, which is uniform on $\{0,\dots,N-1\}$) and using
$\max_i a_i \le \sum_i a_i$ for nonnegative $a_i$, we obtain
\begin{align*}
\mathbb{E}\big[P^t(J_{\mathrm{val}(x^t)}) \,\big|\, \text{past transcript up to }t-1\big]
&\le \frac{1}{N}\sum_{b=0}^{N-1} \max_{1\le i\le m} Q_i^t(J_b) \\
&\le \frac{1}{N}\sum_{b=0}^{N-1} \sum_{i=1}^m Q_i^t(J_b)
= \frac{1}{N}\sum_{i=1}^m \sum_{b=0}^{N-1} Q_i^t(J_b) \\
&= \frac{1}{N}\sum_{i=1}^m 1
= \frac{m}{N},
\end{align*}
since the intervals $J_0,\dots,J_{N-1}$ form a partition of $[0,1]$, so $\sum_b Q_i^t(J_b)=1$ for each $i$.
Taking expectations over the past transcript yields $\mathbb{E}[\pi_t]\le m/N$.

For the second step, note that conditional on $(P^t,x^t)$ the realized draw satisfies
\[
\mathbb{E}\big[\mathbf{1}[p^t\in J_{\mathrm{val}(x^t)}]\ \big|\ P^t,x^t\big]=P^t(J_{\mathrm{val}(x^t)})=\pi_t,
\]
because $p^t\sim P^t$. Taking expectations and summing over $t$ gives the claim.
\end{proof}

\begin{lemma}[Misses force squared loss]
\label{lem:oracle-bits-miss-mse}
Under $\mathcal{D}_{T,N}$, for every realization we have
\[
\sum_{t=1}^T (p^t-y^t)^2
\ \ge\
\frac{1}{4N^2}\sum_{t=1}^T \mathbf{1}\big[p^t\notin J_{\mathrm{val}(x^t)}\big].
\]
Consequently,
\[
\mathbb{E}\Big[\sum_{t=1}^T (p^t-y^t)^2\Big]
\ \ge\
\frac{1}{4N^2}\left(1-\frac{m}{N}\right)T.
\]
\end{lemma}
\begin{proof}
Fix a round $t$. Under $\mathcal{D}_{T,N}$, the label is $y^t=\mu(x^t)$, which is the midpoint of the interval
$J_{\mathrm{val}(x^t)}$. The distance from the midpoint of an interval of width $1/N$ to the complement of the interval
is exactly $1/(2N)$, so
\[
p^t\notin J_{\mathrm{val}(x^t)}\quad\Longrightarrow\quad |p^t-y^t|\ge \frac{1}{2N}
\quad\Longrightarrow\quad (p^t-y^t)^2\ge \frac{1}{4N^2}.
\]
Multiplying by the indicator $\mathbf{1}[p^t\notin J_{\mathrm{val}(x^t)}]$ and summing over $t$ gives the first claim.

For the second claim, take expectations and apply Lemma~\ref{lem:oracle-bits-correctmass}:
\[
\mathbb{E}\Big[\sum_{t=1}^T \mathbf{1}[p^t\notin J_{\mathrm{val}(x^t)}]\Big]
=
T-\mathbb{E}\Big[\sum_{t=1}^T \mathbf{1}[p^t\in J_{\mathrm{val}(x^t)}]\Big]
\ \ge\ \left(1-\frac{m}{N}\right)T.
\]
\end{proof}

Thus far we have established that proper oracle reductions with $m \ll N$ must frequently mispredict the true label with non-negligible margin, and hence incur large squared loss. The next lemma establishes that strong multicalibration bounds force small squared loss. 

\begin{lemma}[Squared loss controlled by bit-group calibration]
\label{lem:oracle-bits-mse-to-mcerr}
For every realization under $\mathcal{D}_{T,N}$,
\[
\sum_{t=1}^T (p^t-y^t)^2
\ \le\
\left(2-\frac{1}{2N}\right)\,\mathrm{MCerr}_T(G_{\mathrm{bits}})
\ <\ 2\,\mathrm{MCerr}_T(G_{\mathrm{bits}}).
\]
\end{lemma}
\begin{proof}
For each realized prediction value $v\in V_T$, let $S_v:=\{t\in\{1,\dots,T\}: p^t=v\}$ and note that
\[
\sum_{t=1}^T (p^t-y^t)^2
= \sum_{v\in V_T}\ \sum_{t\in S_v} (v-y^t)^2.
\]
Fix $v\in V_T$. For each $t\in S_v$ we have the identity
\[
(v-y^t)^2 = (v-y^t)\,v - (v-y^t)\,y^t,
\]
so summing over $t\in S_v$ gives
\[
\sum_{t\in S_v}(v-y^t)^2
= v\sum_{t\in S_v}(v-y^t)\ -\ \sum_{t\in S_v} y^t\,(v-y^t).
\]
Since $0\le v\le 1$, the triangle inequality yields
\begin{equation}\label{eq:oracle-bits-mse-step1}
\sum_{t\in S_v}(v-y^t)^2
\ \le\
\left|\sum_{t\in S_v}(v-y^t)\right|\ +\ \left|\sum_{t\in S_v} y^t\,(v-y^t)\right|.
\end{equation}

Next we expand $y^t=\mu(x^t)$ in terms of the bit groups.
Because $N=2^k$ and $\mathrm{val}(x^t)=\sum_{r=1}^k x_r^t\,2^{k-r}$, we have
\[
y^t=\mu(x^t)
=\frac{\mathrm{val}(x^t)+\tfrac12}{N}
=\frac{1}{2N}+\frac{\mathrm{val}(x^t)}{2^k}
=\frac{1}{2N}+\sum_{r=1}^k 2^{-r}\,x_r^t.
\]
Therefore
\begin{align*}
\sum_{t\in S_v} y^t\,(v-y^t)
&=
\frac{1}{2N}\sum_{t\in S_v}(v-y^t)
+\sum_{r=1}^k 2^{-r}\sum_{t\in S_v} x_r^t\,(v-y^t),
\end{align*}
and applying the triangle inequality gives
\begin{equation}\label{eq:oracle-bits-mse-step2}
\left|\sum_{t\in S_v} y^t\,(v-y^t)\right|
\ \le\
\frac{1}{2N}\left|\sum_{t\in S_v}(v-y^t)\right|
+\sum_{r=1}^k 2^{-r}\left|\sum_{t\in S_v} x_r^t\,(v-y^t)\right|.
\end{equation}

Now relate these terms to calibration error.
By definition, for a group $g$ and a value $v\in V_T$,
\[
B_T(v,g):=\sum_{t=1}^T \mathbf{1}[p^t=v]\;g(x^t)\,(p^t-y^t)
=\sum_{t\in S_v} g(x^t)\,(v-y^t),
\]
and $\mathrm{Err}_T(g)=\sum_{v\in V_T}|B_T(v,g)|$.
Thus
\[
\sum_{t\in S_v}(v-y^t)=B_T(v,g_0),
\qquad
\sum_{t\in S_v} x_r^t\,(v-y^t)=B_T(v,g_r).
\]
Plugging these into \eqref{eq:oracle-bits-mse-step1}--\eqref{eq:oracle-bits-mse-step2} yields
\[
\sum_{t\in S_v}(v-y^t)^2
\ \le\
\left(1+\frac{1}{2N}\right)|B_T(v,g_0)| + \sum_{r=1}^k 2^{-r}\,|B_T(v,g_r)|.
\]
Summing over $v\in V_T$ gives
\[
\sum_{t=1}^T (p^t-y^t)^2
\ \le\
\left(1+\frac{1}{2N}\right)\mathrm{Err}_T(g_0) + \sum_{r=1}^k 2^{-r}\,\mathrm{Err}_T(g_r).
\]
Finally, since $\sum_{r=1}^k 2^{-r}=1-2^{-k}=1-\frac1N$ and each $\mathrm{Err}_T(g_r)\le \mathrm{MCerr}_T(G_{\mathrm{bits}})$,
we obtain
\[
\sum_{t=1}^T (p^t-y^t)^2
\ \le\
\left(1+\frac{1}{2N}+1-\frac{1}{N}\right)\mathrm{MCerr}_T(G_{\mathrm{bits}})
=
\left(2-\frac{1}{2N}\right)\mathrm{MCerr}_T(G_{\mathrm{bits}}),
\]
as claimed.
\end{proof}
It remains to put the pieces together:
\begin{proof}[Proof of Theorem~\ref{thm:oracle-bits}]
Combine Lemma~\ref{lem:oracle-bits-miss-mse} and Lemma~\ref{lem:oracle-bits-mse-to-mcerr}.
Lemma~\ref{lem:oracle-bits-mse-to-mcerr} implies pathwise
\[
\mathrm{MCerr}_T(G_{\mathrm{bits}})
\ \ge\ \frac{1}{2}\sum_{t=1}^T (p^t-y^t)^2,
\]
since $2-\frac{1}{2N}<2$.
Taking expectations and applying Lemma~\ref{lem:oracle-bits-miss-mse} gives
\[
\mathbb{E}\big[\mathrm{MCerr}_T(G_{\mathrm{bits}})\big]
\ \ge\
\frac{1}{2}\,\mathbb{E}\Big[\sum_{t=1}^T (p^t-y^t)^2\Big]
\ \ge\
\frac{1}{2}\cdot \frac{1}{4N^2}\left(1-\frac{m}{N}\right)T
=
\frac{1}{8}\left(1-\frac{m}{N}\right)\frac{T}{N^2}.
\]
If $m\le N/2$ then $(1-m/N)\ge 1/2$, yielding
\[
  \mathbb{E}[\mathrm{MCerr}_T(G_{\mathrm{bits}})]\ \ge\ \frac{T}{16N^2}.
\]

For the parametric statement in the theorem, fix any constant $\gamma\in(0,\tfrac12)$ and set
\[
  k := \big\lfloor \gamma \log_2 T \big\rfloor,\qquad N := 2^k.
\]
Then $|G_{\mathrm{bits}}| = k+1 = \Theta(\log T)$, and
\[
  N \le 2^{\gamma \log_2 T} = T^{\gamma}
\]
implies
\[
  \frac{T}{N^2} \ \ge\ \frac{T}{T^{2\gamma}} \ =\ T^{1-2\gamma}.
\]
If in addition $m=\mathrm{poly}(|G_{\mathrm{bits}}|)$, then since $N = 2^k = T^{\Theta(\gamma)}$ grows superpolynomially
in $|G_{\mathrm{bits}}|$, there exists $T_0=T_0(\gamma)$ such that $m\le N/2$ for all $T\ge T_0$. For such $T$ we obtain
\[
  \mathbb{E}[\mathrm{MCerr}_T(G_{\mathrm{bits}})]\ \ge\ \frac{1}{16}\,T^{1-2\gamma},
\]
as claimed in Theorem~\ref{thm:oracle-bits}.
\end{proof}

\paragraph{Interpretation}
Since $|G_{\mathrm{bits}}|=k+1$ and $N=2^k$, the bound in Theorem~\ref{thm:oracle-bits} can be rewritten as
\[
\mathbb{E}[\mathrm{MCerr}_T(G_{\mathrm{bits}})]
\ \ge\ \frac{1}{8}\left(1-\frac{m}{2^{|G_{\mathrm{bits}}|-1}}\right)\frac{T}{2^{2(|G_{\mathrm{bits}}|-1)}}
\ =\ \Omega\!\left(\left(1-\frac{m}{2^{|G|-1}}\right)\frac{T}{2^{2(|G|-1)}}\right),
\]
where $|G|:=|G_{\mathrm{bits}}|$. In particular, making this lower bound vacuous requires
$m\approx N = 2^{\Theta(|G_{\mathrm{bits}}|)}$, so any proper black-box reduction using
$m=\mathrm{poly}(|G|)$ copies (e.g.\ one copy per group) fails on this instance for large $T$.

A complementary  view is obtained by optimizing $N$ as a function of $m$.
Fix $T\ge1$ and $m\ge1$, and choose
\[
k := \big\lceil \log_2(2m) \big\rceil,\qquad N := 2^k.
\]
Then $2m\le N<4m$, and for any proper $m$-copy reduction $B$ with context-blind oracle $A$ run on $\mathcal{D}_{T,N}$,
Theorem~\ref{thm:oracle-bits} gives
\[
\mathbb{E}[\mathrm{MCerr}_T(G_{\mathrm{bits}})]
\ \ge\ \frac{1}{8}\left(1-\frac{m}{N}\right)\frac{T}{N^2}
\ \ge\ \frac{1}{8}\cdot\frac{1}{2}\cdot\frac{T}{(4m)^2}
\ =\ \frac{T}{256\,m^2}
\ =\ \Omega\!\left(\frac{T}{m^2}\right).
\]
Here $|G_{\mathrm{bits}}| = k+1 = \Theta(\log m)$, so any proper reduction with $m$ context-blind oracle copies can
be forced to incur multicalibration error $\Omega(T/m^2)$ on only $O(\log m)$ groups. This rules out any non-trivial reduction from multicalibration to marginal calibration in the ``sleeping experts style'' (as in \cite{blum2020advancing,acharyaoracle}) which use one oracle per group: $m = \Theta(|G|)$, and forces non-trivial reductions to choose $m$ growing polynomially with $T$.

\end{document}